\documentclass{article}
\usepackage[utf8]{inputenc}

\usepackage[margin=1.3in]{geometry}
\usepackage[utf8]{inputenc}
\usepackage{amsmath}
\usepackage{graphicx}
\usepackage{amsfonts}
\usepackage{subfigure}
\usepackage{float}
\usepackage{algorithm,algorithmicx}
\usepackage[noend]{algpseudocode}

\usepackage{amsthm}
\newtheorem{theorem}{Theorem}
\newtheorem{definition}{Definition}
\newtheorem{claim}{Claim}

\newtheorem{lemma}{Lemma}

\newtheorem{example}{Example}

\usepackage{color}

\newcommand{\poly}{\text{poly}}
\newcommand{\abs}[1]{|#1|}
\newcommand{\card}[1]{|#1|}
\newcommand{\oracle}{\mathcal{O}}
\newcommand{\E}{\mathbf{E}}
\newcommand{\R}{\mathbb{R}}
\newcommand{\N}{\mathbb{N}}
\newcommand{\Dn}{D^\text{n}}
\newcommand{\Dc}{D^\text{c}}
\newcommand{\stt}{\textnormal{s.t.}}

\newenvironment{cpf}
{\begin{trivlist} \item[] {\em Proof of Claim. }}
	{$\hfill\diamond$ \end{trivlist}}

\title{Clustering with Queries under Semi-Random Noise}

\author{Alberto Del Pia
	\thanks{Department of Industrial and Systems Engineering \& Wisconsin Institute for Discovery,
		University of Wisconsin-Madison, Madison, WI, USA.
		E-mail: {\tt delpia@wisc.edu}.}
	\and
	Mingchen Ma
	\thanks{Department of Computer Sciences,
		University of Wisconsin-Madison, Madison, WI, USA.
		E-mail: {\tt mma54@wisc.edu}.}
	\and 
	Christos Tzamos
	\thanks{Department of Computer Sciences,
		University of Wisconsin-Madison, Madison, WI, USA.
		E-mail: {\tt tzamos@wisc.edu}.}
}

\begin{document}
	
	\maketitle
	\begin{abstract}
	The seminal paper by Mazumdar and Saha \cite{mazumdar2017clustering} introduced an extensive line of work on clustering with noisy queries. Yet, despite significant progress on the problem, the proposed methods depend crucially on knowing the exact probabilities of errors of the underlying fully-random oracle. 
	In this work, we develop robust learning methods that tolerate general semi-random noise obtaining qualitatively the same guarantees as the best possible methods in the fully-random model.
	
	More specifically, given a set of $n$ points with an unknown underlying partition, we are allowed to query pairs of points $u,v$ to check if they are in the same cluster, but with probability $p$, the answer may be adversarially chosen. We show that information theoretically $O\left(\frac{nk \log n} {(1-2p)^2}\right)$ queries suffice to learn any cluster of sufficiently large size. Our main result is a computationally efficient algorithm that can identify large clusters with $O\left(\frac{nk \log n} {(1-2p)^2}\right) + \text{poly}\left(\log n, k, \frac{1}{1-2p} \right)$ queries, matching the guarantees of the best known algorithms in the fully-random model. As a corollary of our approach, we develop the first parameter-free algorithm for the fully-random model, answering an open question in \cite{mazumdar2017clustering}.
	\end{abstract}

	\section{Introduction}
	
In a typical graph clustering problem, we are given a graph $G=(V,E)$ and we want to partition the vertices $V$ into several clusters that satisfy certain properties. Clustering is ubiquitous in machine learning, theoretical computer science and statistics as this simple formulation has many applications in both theory and practice. Many theoretical problems can be formulated as clustering and it is a common NP-complete problem. Moreover, many practical problems where we want to put data or people that are similar together can be viewed as clustering, like record linkage and entity resolution \cite{fellegi1969theory} in databases, or community detection in social networks. 

However, in many applications, one may not have direct access to the full graph, and it may be costly to query the relationship between two nodes. For example, in entity resolution where the goal is to find records in a database that refer to the same underlying entity, it is common to use crowdsourcing to develop human in the loop systems for labeling the edges \cite{green2020clustering,karger2011iterative,wang2012crowder,dalvi2013aggregating,gokhale2014corleone,vesdapunt2014crowdsourcing,mazumdar2017theoretical}. Asking humans requires effort, time, and money, so one would want to cluster the items efficiently without asking workers to compare every pair of items. 


Motivated by these applications, the work of \cite{mazumdar2017clustering}
introduced a theoretical model of clustering with queries. 
In this model, we don't have access to the edges of the graph in advance but may query a similarity oracle that may not always give the correct answer. The problem is defined as follows.

\begin{definition}\label{model 1}(Clustering with a faulty oracle)
	We are given a set of points $V = [n]$, which contains $k$ latent clusters $V^*_i \subseteq V$ for $i \in [k]$ such that $\bigcup_{i \in [k]} V^*_i = V$ and for every $1 \le i<j \le k$, $V^*_i \cap V^*_j = \emptyset$. For every pair of points $u,v \in V$, the edge $(u,v)$ is labeled by $1$, if $u,v$ are in the same cluster, and is labeled by $0$, if $u,v$ are in different clusters.
	The number $k$ and the clusters $V^*_i$ for $ i \in [k]$ are unknown. We are given an oracle $\oracle: V \times V \to \{0,1\}$ to query point pairs of $V$. Every time the oracle $\oracle$ is invoked it takes a pair of points $u,v$ as an input and outputs a label for the edge $(u,v)$ which may be erroneous. Our goal is to recover the latent clustering with high probability, minimizing the queries to the oracle $\oracle$.   
\end{definition}

While the above formulation of Definition~\ref{model 1} does not specify how the errors are introduced by the oracle, the work of \cite{mazumdar2017clustering} focused specifically on a fully-random faulty oracle that gives incorrect answers with a fixed probability of error $p \in [0,1/2)$ known in advance.

\begin{definition}
	(Fully-random model of clustering with a faulty oracle)
	Under the fully-random model, the oracle $\oracle$ of Definition~\ref{model 1} behaves as follows. There is a known error parameter $p \in [0,1/2)$ such that, for every point pair $u,v$, $\oracle(u,v)$ outputs the true label of $(u,v)$ with probability $1-p$ and outputs the wrong label of $(u,v)$ with probability $p$.
\end{definition}

\cite{mazumdar2017clustering} showed that in this model $\Omega\left(\frac{nk}{(1-2p)^2}\right)$ queries are necessarily needed to recover the underlying clustering while $O\left(\frac{nk\log n}{(1-2p)^2}\right)$ queries suffice to learn any large enough cluster. They also designed a computationally-efficient algorithm with query complexity $O\left(\frac{nk^2\log n}{(1-2p)^4}\right)$ to recover large clusters. Since then, follow-up papers \cite{green2020clustering,pmlr-v134-peng21a} extended their results and obtained efficient algorithms with lower query complexity $O\left(\frac{nk \log n} {(1-2p)^2}\right) + \text{poly}\left(\log n, k, \frac{1}{1-2p} \right)$. 

While these works pin down the query complexity of the problem in the fully-random model, they crucially rely both on the fully-random model and the exact knowledge of the error probability parameter $p$. In fact, \cite{mazumdar2017clustering} posed as an open problem whether one can design a parameter-free algorithm with the same guarantees.

Motivated by these shortcomings of the fully-random model, our goal in this work is to obtain more robust algorithms that work beyond the fully-random model and do not rely on the knowledge of the error probabilities. Removing these crucial assumptions will enable broader applicability of the algorithms as in practice, the amount of noise may depend on the particular nodes being compared and may vary significantly from query to query making it impossible to know it or predict it in advance. 

Our work focuses on a significantly more general semi-random model that allows the oracle answers to be given adversarially with some probability.


\begin{definition}\label{def semi}(Semi-random model of clustering with a faulty oracle)
	Under Definition~\ref{model 1}, the oracle $\oracle$ is defined in the following way. There is a known error parameter $p \in (0,1/2)$ such that, for every point pair $u,v$, with probability $1-p$, $\oracle(u,v)$ outputs the true label of $(u,v)$ and with probability $p$, $(u,v)$ is corrupted and $\oracle(u,v)$ outputs an arbitrary label given by an adversary, depending on the sequential output of $\oracle$ and the latent clustering.
\end{definition}

An important special case of the semi-random model corresponds to the case where 
every edge $(u,v)$ has a fixed and unknown probability of error $p_{(u,v)}$ chosen in advance that is upper bounded by the known bound $p$. We refer to this case as \emph{non-adaptive semi-random noise} and note that the more general Definition~\ref{def semi} can handle more adaptive instantiations of noise where the answers of the oracle may depend on the answers given in advance.

The main contribution of our work is the design of novel algorithms that can tolerate semi-random noise essentially matching the guarantees obtained for the fully-random model. Before formally presenting our results, we give an overview of the existing methods and guarantees for the fully-random model.






\subsection{Prior work on the fully-random model}

In Table~\ref{table 1}, we summarize previous results, as well as our main results, for the fully-random model and the semi-random model.
\begin{table}[H]\centering
	\begin{tabular}{|c|c|c|c|}
		\hline   \textbf{Model} & \textbf{Query Complexity} & \textbf{Reference} & \textbf{Remark} \\   
		\hline Fully-random & $\Omega\left(\frac{nk}{\left(1-2p\right)^2}\right)$ &  \cite{mazumdar2017clustering} & Lower bound  \\ 
		& $O\left(\frac{nk\log n}{\left(1-2p\right)^2}\right)$ & \cite{mazumdar2017clustering}   & Time inefficient\\
		& $O\left(\frac{nk^2\log n}{\left(1-2p\right)^4} \right)$ &  \cite{mazumdar2017clustering} & \\  
		& $O\left(\frac{n\log n}{\left(1-2p\right)^2}+ \frac{\log^2n}{\left(1-2p\right)^6}  \right)$ &  \cite{green2020clustering}   & $k=2$\\
		& $O\left(\frac{nk\log n}{\left(1-2p\right)^2}+ \frac{k^4 \log^2n}{\left(1-2p\right)^4}  \right)$ &   \cite{pmlr-v134-peng21a} & Nearly balanced    \\
		& $O\left(\frac{nk\log n}{\left(1-2p\right)^2}+ \frac{k^{10} \log^2n}{\left(1-2p\right)^4}  \right)$ &  \cite{pmlr-v134-peng21a} &  $k$ is known\\
		& $O\left(\frac{nk\log n}{\left(1-2p\right)^2}+ \frac{k^9\log k \log^2 n}{\left(1-2p\right)^{12}} \right)$ & Theorem~\ref{th fully} & Parameter-free \\
		\hline   Semi-random & $O \left( \frac{nk\log n}{\left(1-2p\right)^2} \right)$ & Theorem~\ref{th estimation semi} & Time inefficient\\
		& $O\left(\frac{nk\log n}{\left(1-2p\right)^2}+ \frac{k^9\log k \log^2 n}{\left(1-2p\right)^{12}} \right)$ & Theorem~\ref{th final} & \\
		& $O\left(\frac{n\log n}{\left(1-2p\right)^2}+ \frac{\log^2n}{\left(1-2p\right)^6}\right)$& Theorem~\ref{th k=2} & $k=2$  \\
		\hline
	\end{tabular}
	\caption{Query complexity of algorithms under the fully-random and the semi-random model}
	\label{table 1}   
\end{table}

Previous work that studied the query complexity of the clustering problem focused on the fully-random model. \cite{mazumdar2017clustering} gave an information theoretical algorithm that can recover all clusters of size $\Omega\left(\frac{\log n}{(1-2p)^2}\right)$ with query complexity $O\left(\frac{nk\log n}{(1-2p)^2}\right)$, which matches the information theoretical lower bound of $\Omega\left(\frac{nk}{(1-2p)^2}\right)$ in the same paper within a $O(\log n)$ factor. They also designed an efficient algorithm that can recover all clusters of size at least $\Omega(\frac{k\log n}{(1-2p)^4})$ with query complexity $O(\frac{nk^2\log n}{(1-2p)^4})$. Followed by their work, \cite{green2020clustering} gave an efficient algorithm with an improved query complexity $O\left(\frac{n\log n}{\left(1-2p\right)^2}+ \frac{\log^2n}{\left(1-2p\right)^6}  \right)$.
More recently, \cite{pmlr-v134-peng21a} designed an efficient algorithm that recovers all clusters of size $\Omega\left(\frac{k^4\log n}{(1-2p)^2}\right)$ with query complexity $O\left(\frac{nk\log n}{\left(1-2p\right)^2}+ \frac{k^{10} \log^2n}{\left(1-2p\right)^4}  \right)$ for known $k$. For every constant $k$, their query complexity matches the information lower bound within a $O(\log n)$ factor. Their algorithm can even exactly recover the underlying clustering with query complexity $O\left(\frac{nk\log n}{\left(1-2p\right)^2}+ \frac{k^4 \log^2n}{\left(1-2p\right)^4}  \right)$ if each underlying cluster has size $\Omega\left(\frac{n}{k}\right)$.

\subsection{Our contributions}
We now present our contributions for the semi-random model in more detail.

\paragraph{An information-theoretically tight algorithm} We first give an information theoretical algorithm for the problem presented in Section~\ref{sec information}.

\begin{theorem}\label{th estimation semi}
	There is an algorithm \textsc{Estimation}$(V,p)$ such that	under the semi-random model, \textsc{Estimation}$(V,p)$ has query complexity $O\left(\frac{nk\log n}{(1-2p)^2}\right)$ and recovers all clusters of size at least $\Omega\left(\frac{\log n}{(1-2p)^2}\right)$ with probability at least $1-1/\poly(n)$.
\end{theorem}

Theorem~\ref{th estimation semi} shows even under the semi-random model, $O\left(\frac{nk\log n}{(1-2p)^2}\right)$ queries suffice to learn all clusters of size $\Omega\left(\frac{\log n}{(1-2p)^2}\right)$. This matches the performance of the information theoretical algorithm proposed in \cite{mazumdar2017clustering}. Furthermore, since the fully-random model is a special case of our semi-random model and the information theoretical lower bound for the fully-random model is $\Omega\left(\frac{nk}{(1-2p)^2}\right)$, our query complexity matches the information theoretical lower bound within a $O(\log n)$ factor.

While Theorem~\ref{th estimation semi} gives a nearly-tight information theoretical bound for the problem, the underlying algorithm is not computationally efficient. This is expected as there is a conjectured computational-statistical gap even in the case of fully-random noise \cite{pmlr-v134-peng21a}.

\paragraph{A computationally efficient algorithm}  We next turn to the question of what can be achieved using a computationally efficient algorithm. We obtain the following performance guarantee.

\begin{theorem} \label{th final}
	There is an algorithm \textsc{Clustering}$(V,p)$, such that under the semi-random model, with probability at least $1-1/\poly(n)$, \textsc{Clustering}$(V,p)$ recovers all $V^*_i$, such that $\card{V^*_i} = \Omega\left(\frac{k^4\log n}{\left(1-2p\right)^{6}}\right) $ in polynomial time. Furthermore, the query complexity of \textsc{Clustering}$(V,p)$ is $O\left(\frac{nk\log n}{\left(1-2p\right)^2}+ \frac{k^9\log k \log^2 n}{\left(1-2p\right)^{12}} \right)$.
\end{theorem}

Our algorithm, presented in Section~\ref{sec efficient}, can recover all large clusters under the semi-random model with a query complexity of $O\left(\frac{nk \log n} {(1-2p)^2}\right) + \text{poly}(k, 1/(1-2p), \log n )$. This bound qualitatively matches the best known bound from \cite{pmlr-v134-peng21a} for the fully-random model, and even achieves a slightly better dependence on $k$. We note that a bound of $\Omega\left(\frac{nk} {(1-2p)^2}\right) + \text{poly}( k, 1/(1-2p) )$ is conjectured by \cite{pmlr-v134-peng21a} to be necessary for computationally efficient estimation even in the fully-random model.


\paragraph{A parameter-free algorithm for the fully-random model}  As a corollary of our approach, we design the first efficient parameter-free algorithm under the fully-random model whose performance is given by the following theorem and solves the open question given by \cite{mazumdar2017clustering}.
\begin{theorem}\label{th fully}
	Under the fully-random model, there is a parameter-free algorithm such that with probability at least $1-1/\poly(n)$,  recovers all clusters of size at least $\Omega\left(
	\frac{k^4\log n }{\left(1-2p\right)^6}  \right)$. Furthermore, the query complexity of the algorithm is $O\left(\frac{nk\log n}{\left(1-2p\right)^2} + \frac{k^9\log k\log^2n}{\left(1-2p\right)^{12}} \right)$.
\end{theorem}

\subsection{Technical overview}

The main approach in developing algorithms with low query complexity is to first identify a small, but large enough, subset of vertices $B$ that mostly come from the same cluster $V^*_i$ and then compare all vertices in the graph to the vertices of $B$ to fully identify the whole cluster $V^*_i$ with high probability. Such a set $B$ is called biased, and is computed by first subsampling a subgraph $T$ of the whole graph and solving a clustering problem in the subgraph. Then, once we identify a cluster $V^*_i$, we can repeat the process to recover the remaining clusters as well.
This is a common technique of the prior work \cite{mazumdar2017clustering,green2020clustering,pmlr-v134-peng21a} as well as our work. The main challenge which leads to the difference between the methods is how one can arrive at such a biased set.

To get an information theoretical algorithm, \cite{mazumdar2017clustering} found the largest subcluster of $T$ by computing the heaviest subgraph of $T$. However, as we show in Appendix~\ref{apendix ex3}, this method fails under the semi-random model even if we have two clusters. 
To get an efficient algorithm,
\cite{pmlr-v134-peng21a} did this by filtering small subclusters of $T$ via counting degree of each vertex and running an algorithm proposed by \cite{vu2018simple} for a community detection problem under the Stochastic Block Model, which highly depends on the fully-random noise. On the other hand, \cite{mazumdar2017clustering,green2020clustering} used a simple disagreement counting method to cluster the subgraph $T$. While this simple technique is again very tailored to the fully-random model, we can extend this to the semi-random model but only in a very special case. We obtain an algorithm for semi-random noise where there are  $k=2$ clusters and the noise is non-adaptive (see Theorem~\ref{th k=2} in Appendix~\ref{apendix bicluster}). As we show, this technique breaks down completely once any of these two restricting assumptions are removed. 
In general, previous efficient algorithms on fully random models can fail easily under semi-random models, because they
all use techniques such as counting disagreements or counting degrees locally to obtain information from a single vertex or a pair of vertices. These statistics highly depend on the exact knowledge of the noise rate and thus under the semi-random model, an adversary can easily make the algorithms fail. A detailed discussion can be found in Appendix~\ref{apendix sec3}. To obtain more robust efficient algorithms under the semi-random model, a key challenge is to design a statistic that can obtain information from a larger neighborhood of vertices and can be computed efficiently.

\paragraph{Our Approach}

To obtain robust algorithms for clustering under more than $2$ clusters and more general semi-random noise we require a more involved clustering procedure for the subsampled graph $T$ which we carefully choose. 

For our information theoretical algorithm, our method computes the largest subset of $T$ that has no negative cut (assuming edges that are labeled 0 contribute as -1). As we show, such a set must correspond a set of vertices all coming from the same cluster in the underlying partition, provided that $T$ is large enough. 

As this step is computationally intractable, to obtain a computationally-efficient algorithm, our method relies on efficiently computing an (approximate) correlation clustering of $T$. 
Our key observation is that when $T$ is large enough, every clustering that has a small cost must be close to the underlying clustering and must have a special structure. To make this more specific, such a clustering function must contain some very large cluster and each of these large clusters must be biased to contain a majority of points from the same true cluster. This implies if we can compute a correlation clustering $\tilde{T}$ of $T$ then we can use those large clusters in $\tilde{T}$ to recover the corresponding underlying clusters.


To obtain an approximation to the correlation clustering of the sample set $T$, we rely on an approximation algorithm developed by \cite{mathieu2010correlation,ailon2008aggregating} that obtains an SDP relaxation of the clustering problem and then performs a rounding step. We show that the resulting clustering that the algorithm obtains has a sufficiently small an additive error $O\left( \frac{\card{T}^{3/2}}{\left(1-2p\right)} \right)$ that enables us to identify heavily biased clusters efficiently. By carefully choosing the size of $T$, we show that with high probability, the clustering we obtain must contain at least one big cluster, which is a biased set for a true cluster $V^*_i$ for $i \in [k]$.

\subsection{Further related work}
There has been a lot of work in developing algorithms for clustering.
A lot of research has focused specifically on clustering under random graphs. Typical problems include community detection under stochastic block models (SBM) \cite{abbe2017community} and clique detection under planted clique models \cite{alon1998finding}. In these problems, a hidden structure such as a clustering or a clique is planted in advance, a random graph is generated according to some distribution and we are asked to recover the hidden structure efficiently using the given random graph. 

Another popular clustering problem is correlation clustering, which was proposed in \cite{bansal2004correlation}. In this problem, we are given an undirected graph $G$ and our goal is to partition the vertices into clusters so that we minimize the number of disagreements or maximize the number of agreements. As the correlation clustering problem is NP-hard and many works develop efficient approximation algorithms \cite{bansal2004correlation,demaine2003correlation,giotis2006correlation,swamy2004correlation,charikar2005clustering,ailon2012note,makarychev2015correlation,mathieu2010correlation} for worst case instances, while others  \cite{shamir2007improved,joachims2005error} focus on the average case complexity of clustering when the graph is generated according to some underlying distribution.

A popular application of clustering is the signed edge prediction problem \cite{leskovec2010predicting,burke2008mopping,brzozowski2008friends,chen2014clustering}. In this problem, we are given a social network, where each edge is labeled by `+' or `-' to indicate if two nodes have positive relations or negative relations. The goal here is to use a small amount of information to recover the sign of the edges, which implies we want to reconstruct the network by partial information.

Besides the large body of work on clustering problems with access to the full graph, recently other papers studied clustering problems with queries under different settings.  \cite{ashtiani2016clustering,gamlath2018semi} study the k-means problem with same-cluster queries. \cite{saha2019correlation,ailon2018approximate} study the correlation clustering problem with same-cluster queries. Some other recent works on clustering with queries include \cite{huleihel2019same,li2021learning,bressan2020exact}.

Beyond clustering, there are also other settings in learning theory where semi-random noise makes the problem significantly more challenging and requires more sophisticated algorithms than the corresponding fully-random case.
Semi-random noise corresponds to the popular Massart noise model~\cite{massart2006risk} in the context of robust classification. While classification under fully-random noise was known for many years \cite{blum1998polynomial}, robust learning methods that can tolerate Massart noise were only recently discovered \cite{diakonikolas2019distribution, chen2020classification}.

	\section{Preliminaries and Notation}
	Let $V = [n]$ be a set of points, which contains $k$ \emph{underlying clusters} $V^*_i \subseteq V$, for $i \in [k]$, such that $\bigcup_{i \in [k]} V^*_i = V$ and $V^*_i \cap V^*_j = \emptyset$, for every $1 \le i<j \le k$. 
	We say a set $S \subseteq V$ is a \emph{subcluster} if $S \subseteq V^*_i$ for some $i \in [k]$.
	We say $\tilde{V}: V \times V \to \{0,1\}$ is a \emph{clustering function} over $V$ based on $\{\tilde{V_1},\dots,\tilde{V_t}\}$, if $\{\tilde{V_1},\dots,\tilde{V_t}\}$ is a partition of $V$ and, for every $(u,v) \in V \times V$,
	\begin{align*}
		\tilde{V}(u,v) =\begin{cases}
			& 1 \ \text{if }u\in \tilde{V_i}, v \in \tilde{V_j}, i=j,  \\
			& 0 \ \text{if }u\in \tilde{V_i}, v \in \tilde{V_j}, i\neq j.
		\end{cases} 
	\end{align*}
	In particular, throughout the paper, we denote by $V^*$ the clustering function over $V$ based on the underlying clusters and we denote by $\bar V$ the binary function over $V \times V$ corresponding to a realization of $\oracle$ over all point pairs of $V$.
	Given a binary function $F:V \times V \to \{0,1\}$, the \emph{adjacency matrix} of $F$ is the matrix $M(F) \in \{0,1\}^{|V|\times|V|}$, such that $M(F)_{uv}=F(u,v)$ for every $u,v \in V$.
	For convenience, when it does not create confusion, we use the same notation for a clustering function, the set of clusters it is based on, and its adjacency matrix.

	Given $A,B \in \R^{n \times n}$, we define the \emph{distance} between $A,B$ to be $d(A,B) : = \sum_{1 \le i \le j \le n}\abs{A_{ij}-B_{ij}}$.
	Let $F,H$ be two binary functions over $V \times V$. We define the \emph{distance} between $F,H$ to be 
	\begin{align*}
		d(F,H):= d(M(F),M(H)) = \sum_{1\le u \le v \le |V|} \abs{F(u,v)-H(u,v)}.
	\end{align*}
	Given a binary function $E$ over $V \times V$, a \emph{correlation clustering} $\tilde{V}$ of $E$ is a clustering of $V$ that minimizes $d(V',E)$ among all clustering $V'$ of $V$.
	
	Next, we introduce two definitions that will be heavily used throughout the paper.
	
	\begin{definition}\label{def eta}
		Let $\eta \in (0,1/2]$ and $C \subseteq V$. A subset $B$ of $V$ is called an \emph{$(\eta,C)$-biased set} if 
		\begin{align*}
			\card{B \cap C} \ge (\frac{1}{2}+\eta)\card{B}.
		\end{align*} 
	\end{definition}
	
	Intuitively, an $(\eta,C)$-biased set is a set whose majority of points come from $C$. On the other hand, if a set does not contain a significant fraction of points that come from an underlying cluster, we call it an $\eta$-bad set. Formally, we have the following definition.
	
	\begin{definition}
		Let $\eta \in (0,1/2]$. A subset $B$ of $V$ is called an \emph{$\eta$-bad set} if $B$ is not an $(\eta,V^*_i)$-biased set for every $i \in [k]$.
	\end{definition}
	
	The importance of Definition~\ref{def eta} is that, under the semi-random model, we can recover an underlying cluster $V^*_i$ from an $(\eta,V^*_i)$-biased set using the following simple procedure, which has been proposed in \cite{ben1999clustering,mazumdar2017clustering,green2020clustering,pmlr-v134-peng21a}.
	\begin{algorithm}[H]
		\caption{\textsc{DegreeTest}$(v,B)$ (Test if $v \in V^*_i$ using an $(\eta,V^*_i)$-biased set $B$)}\label{alg test}
		\begin{algorithmic}
			\If{$S = \sum_{u \in B}\oracle(u,v) \ge \card{B}/2$}
			\Return ``Yes" \textbf{else}
			\Return ``No"
			\EndIf
		\end{algorithmic}
	\end{algorithm}
	The intuition behind Algorithm~\ref{alg test} is that if more than half of the points of $B$ come from $V^*_i$, then we can use $B$ to distinguish if a point $v$ is in $V^*_i$ or not, by looking at the query results.
	According to \cite{pmlr-v134-peng21a}, for every constant $\eta \in (0,1/2]$, we can use an $(\eta,V^*_i)$-biased set $B$ of size $\Omega(\frac{\log n}{\eta^2(1-2p)^2})$ to recover $V^*_i$ via \textsc{DegreeTest}$(v,B)$ with high probability under the fully-random model. However, under the semi-random model, to recover $V^*_i$ using $B$, $B$ needs to satisfy some additional conditions. To state this formally, we have the following Lemma~\ref{lm test}. We leave the proof to Appendix~\ref{apendix lm1}. We remark that the additional condition $\eta>p$ in Lemma~\ref{lm test} is necessary for \textsc{DegreeTest}$(v,B)$ to succeed.
	
	\begin{lemma}\label{lm test}
		Under the semi-random model, let $\eta \in (p,1/2]$ and let $B$ be an $(\eta,V^*_i)$-biased set for some $i \in [k]$. If $\card{B} \ge \max \{\frac{80\log n}{\eta^2(1-2p)^2} ,\frac{5\log n}{\left(\eta-p\right)^2}\}$, then with probability $1 - 1/\poly(n)$, for every $v \in V$, \textsc{DegreeTest}$(v,B)$ returns ``Yes" if $v \in V^*_i$, and it returns ``No" if $v \not \in V^*_i$.
	\end{lemma}

	\section{Information theoretical algorithm}
	\label{sec information}
	Before designing efficient algorithms, we first need to figure out how many queries are needed in order to recover the underlying clusters under the semi-random model.
	In this section, we answer this question formally and we propose an information theoretical algorithm.
	Our algorithm has a similar structure to the information theoretical algorithm in \cite{mazumdar2017clustering}, but we use a different statistic to overcome the semi-random noise. 
	In particular, our algorithm can achieve query complexity $O\left(\frac{nk\log n}{\left(1-2p\right)^2}\right)$ under the semi-random model, which matches the information theoretical lower bound  $\Omega\left(\frac{nk}{\left(1-2p\right)^2}\right)$ within a $O(\log n)$ factor. 
	Our main algorithm is Algorithm~\ref{alg information}. The theoretical guarantee of Algorithm~\ref{alg information} is stated in Theorem~\ref{th estimation semi} presented in the introduction. 
	The proof of Theorem~\ref{th estimation semi} is in  Appendix~\ref{apendix information}.


	\begin{algorithm}
		\caption{\textsc{Estimation}$(V,p)$ (Recover all large clusters of $V$)}\label{alg information}
		\begin{algorithmic}
			\State Let $C = \emptyset$	
			\State Randomly select $T \subseteq V$ with $\card{T} = \frac{c\log n}{\left(1-2p\right)^2}$, $V \gets V \setminus T$ \Comment{$c$ is a large enough constant}
			\While{$V \neq \emptyset$}
			\While{$\textsc{FindBigClusters}(T) = \emptyset$} \Comment{Find subsets $T \cap V^*_i, i \in [k]$ of size $\Omega\left(\frac{\log n}{(1-2p)^2}\right)$.}
			\State Randomly select $v$ from $V$, $T \gets T \cup \{v\}, V \gets V \setminus T$
			\EndWhile
			\For{$A \in \textsc{FindBigClusters}(T)$}
			\State Randomly select $B \subseteq A$, such that $\card{B} = \frac{320\log n}{\left(1-2p\right)^2}$
			\State $A \gets A \cup \{v \in V \mid \textsc{DegreeTest}(v,b) = \text{``Yes"}\}$
			\State $C \gets C \cup \{A\}$, $V \gets V \setminus A$
			\EndFor

			\EndWhile

			\State \Return $C$	
		\end{algorithmic}
	\end{algorithm}

	\textsc{Estimation}$(V,p)$ outputs a set of clusters $C$. Each element in $C$ is an underlying cluster. Each point $v \in V$ is a point that we cannot assign to a cluster in $C$. 
	In the algorithm, we maintain a set of points $T$ as a sample set. If we can find all sets of the form $T^*_i=T \cap V^*_i$ such that $\card{T^*_i} = \Omega\left(\frac{\log n}{(1-2p)^2}\right)$, then we can use $T^*_i$ to recover $V^*_i$ with high probability, according to Lemma~\ref{lm test}. 
	If such $T_i$ does not exist, we enlarge $T$ until there is such a set. 
	In this way, we can recover all large underlying clusters.
	To find these sets $T^*_i$, we an use the following Algorithm~\ref{alg Process} with unlimited computational power.
	
	\begin{algorithm}[H]
		\caption{\textsc{FindBigClusters}$(T)$ (Extract all subsets $T \cap V^*_i, i \in [k]$, of large size)}\label{alg Process}
		\begin{algorithmic}
			\State Query every point pair in $T$ and assign weight $w_{uv}=2\oracle(u,v)-1$ to each point pair
			\State Let $C=\emptyset$
			
			\While{$\card{T} \ge \frac{320\log n}{\left(1-2p\right)^2}$}
			\State Find the largest subset $S\subseteq T$ such that $val_S := \min_{A \subseteq S} \sum_{u \in A}\sum_{v \in S\setminus A}w_{uv} > 0$
			\If{ $\card{S} < \frac{320\log n}{\left(1-2p\right)^2}$ }
			\Return $C$
			\EndIf
			
			\State $T \gets T \setminus S, C \gets C \cup \{S\}$

			\EndWhile
			
			\State \Return $C$
		\end{algorithmic}
	\end{algorithm}

	
	\textsc{FindBigClusters}$(T)$ assigns a weight $w_{uv}=2\oracle(u,v)-1$ to each point pair and extracts the largest subset $S \subseteq T$ such that $S$ has no negative cut. We summarize the theoretical guarantee of Algorithm~\ref{alg Process} via the following Theorem~\ref{th Process}, which plays a key role in the proof of Theorem~\ref{th estimation semi}.

	\begin{theorem}\label{th Process}
		Let $T \subseteq V$ be a set of points. Under the semi-random model, with probability at least $1-1/\poly(n)$,  $\textsc{FindBigClusters}(T)=\{T \cap V^*_i \mid \card{T \cap V^*_i}  \ge \frac{320 \log n}{(1-2p)^2}\}$.
	\end{theorem}
	
	We sketch the proof of Theorem~\ref{th Process} here.
	We will show that if $T$ contains a large subcluster, then with high probability, the largest subcluster will not contain a negative cut. 
	On the other hand, with high probability, any large subset of $T$ that is not a subcluster must contain a negative cut. 
	Therefore, every time we find a large subset that contains no negative cut, we must find the largest subcluster contained in $T$. 
	We summarize the above argument in Lemma~\ref{lm clique} and Lemma~\ref{lm cut} in Appendix~\ref{apendix th2}. 
	A complete proof of Theorem~\ref{th Process} can also be found in Appendix~\ref{apendix th2}.


	
	We remark that this information theoretical result is nontrivial. In our algorithm we process the sampled set $T$ by finding the largest subset that has no negative cut, while in \cite{mazumdar2017clustering}, the authors did this by computing the heaviest subgraph. 
	A simple example with $k=2$ can be used to show that their algorithm fails to recover the underlying clusters under the semi-random model.
	Suppose we have two underlying clusters with the same size. 
	We run the algorithm in \cite{mazumdar2017clustering} to recover the two clusters. 
	Every time we sample a set $T$ of $\Omega(\log n)$ size, the adversary always outputs the true label for $(u,v)$ if $u,v$ are in the same underlying cluster, but outputs a wrong label if $u,v$ are in different underlying clusters. 
	When the noise level is high, in expectation, the heaviest subgraph of $T$ is $T$ itself and we have failed to recover the underlying clusters.
	We present this example in detail in Appendix~\ref{apendix ex3}.

	\section{Computationally efficient algorithm }\label{sec efficient}

	In this section, we develop a computationally efficient algorithm for our clustering problem under semi-random noise, presented in Algorithm~\ref{alg clustering}. 
	We analyze the performance of Algorithm~\ref{alg clustering} in Theorem~\ref{th final} presented in the introduction. The full proof of Theorem~\ref{th final} is in Appendix~\ref{apendix th4}.
	
	\begin{algorithm}[H]
		\caption{\textsc{Clustering}$(V,p)$ (Recover all large clusters in $V$ efficiently)}\label{alg clustering}
		\begin{algorithmic}
			\State Let $C=\emptyset$, $s_t = \frac{c't^3 \log n}{\left(1-2p\right)^{6}}$  \Comment{$c'$ is a large enough constant}
			
			\While{$V \neq \emptyset$}
			\State $t=1, h=0$
			
			\While{$h = 0$ and  $\card{V} \ge ts_t$} 
			\State Randomly select $T \subseteq V$ of size $ts_t$
			\State Let $\tilde{T}=\textsc{ApproxCorrelationCluster}(T,1/\poly(n))$ \Comment{Compute an approximation of the correlation clustering of $T$, w.p. $1-1/\poly(n)$}
			\State Let $\{\hat{T_1},\dots,\hat{T_h}\} : = \{\tilde{T_i} \in \textsc{ApproxCorrelationCluster}(T) \mid \card{\tilde{T_i}} >s_t/2\} $,  $t \gets 2t$
			\EndWhile
			\If{$\card{V} < ts_t$ and $h=0$} 
			\Return $C$ \Comment{Stop when $V$ only contains small clusters}
			\EndIf

			\For{$i \in [h]$} \Comment{Recover underlying clusters via $\eta$-biased sets}
			\State Randomly select $B_i \subseteq \hat{T_i}$ of size $\frac{720 \log n}{\left(1-2p\right)^2}$
			\State Let $\tilde{V_i} = \{v \in V \mid \textsc{DegreeTest}(v,B_i) = \text{``Yes"}\}$,
			$C \gets C \cup \{\tilde{V_i}\}$, $V \gets V \setminus \tilde{V_i}$
			\EndFor

			\EndWhile
			\State \Return $C$

		\end{algorithmic}
	\end{algorithm}

	The output of \textsc{Clustering}$(V,p)$ is a set of underlying clusters $C$. The set $V$ contains points that we have not assigned to a cluster in $C$. In the algorithm, we maintain a variable $t$ to estimate the number of underlying clusters in $V$. 
	In each round, we sample a set of points $T$, whose size depends on $t$, and we compute a clustering $\tilde{T}$ of $T$ to approximate the correlation clustering of $T$ via \textsc{ApproxCorrelationCluster}$(T,1/\poly(n))$. 
	As we will see, when $\card{T}$ is large enough, with high probability, we can find $(\eta,V^*_i)$-biased sets from this approximate correlation clustering. Thus, we can use these biased sets to recover the corresponding underlying clusters. In this way, we can recover all large underlying clusters until $V$ contains a small number of points.

	Next, we present the outline of the remainder of this section. 
	In Section~\ref{sec SDP}, we give \textsc{ApproxCorrelationCluster} and show how well it can approximate the correlation clustering of $T$. 
	In Section~\ref{sec oneround}, we present the structure of the approximate correlation clustering. Finally,  we sketch the proof of Theorem~\ref{th final} in Section~\ref{sec proof}.

	\subsection{Approximate correlation clustering}\label{sec SDP}
	
	Let $T$ be a set of points and let $F$ be a binary function over $T \times T$. We consider the following natural SDP relaxation of the correlation clustering problem, which has been used for designing the approximate algorithm in \cite{mathieu2010correlation}.

	\begin{align}
		\label{pr SDP}
		\tag{SDP(F)}
		\begin{split}
			\min \ & d(X,F)=\sum_{1 \le u \le v \le \card{T}} \abs{X_{uv}-F(u,v)} \\
			\stt \ & X_{uv}+X_{vw}-X_{uw} \le 1 \ \forall u,v,w \in T \\
			& X_{uu}=1 \; \forall u \in T, \ X_{uv} \ge 0 \; \forall u,v \in T \\
			& X \succeq 0.
		\end{split}
	\end{align}
	
	\begin{algorithm}[H]
		\caption{\textsc{ApproxCorrelationCluster}$(T,\delta)$ (Approximate correlation clustering of $T$)}\label{alg one round}
		\begin{algorithmic}
			
			\State Query all point pairs of $T \subseteq V$ and construct the corresponding binary function $\bar{T}$

			\State Compute $X^*$, a near optimal solution to SDP$(\bar{T})$, with additive error at most $1/\poly(n)$
			
			\State Use $X^*$ to do rounding $O\left(\log\frac{1}{\delta} \right)$ times and return the clustering $\tilde{T}$ that minimizes $d(\tilde{T},X^*)$
			
			\State (Compute a clustering $\tilde{T}$ by rounding around $X^*$): Start with $\Tilde{T}=\emptyset$

			
			\While{$T \neq \emptyset$} 
			
			\State Randomly select a point $v$ from $T$ and let $U = \{v\}$
			
			\For{$u \in T \setminus \{v\}$}  add $u$ to $U$ with probability $X^*_{uv}$
			
			\EndFor
			
			\State $\tilde{T} \gets \tilde{T} \cup \{U\}$,  $T \gets T \setminus \{U\}$
			
			\EndWhile
			
			\State \Return $\tilde{T}$

		\end{algorithmic}
	\end{algorithm}

	\begin{theorem}\label{th app}
		Let $T \subseteq V$ and $\delta \in (0,1)$. Then $ \tilde{T}= \textsc{ApproxCorrelationCluster}(T,\delta)$ can be computed in $\poly\left(\card{T},\log\frac{1}{\delta} \right)$ time. Furthermore,
		under the semi-random model, there is a constant $c_1>0$ such that with probability at least $1-O(\delta+\exp\left(-\card{T}\right)$,
		\begin{align*}
			d(\tilde{T},\bar{T}) \le d(T^*,\bar{T}) + \frac{c_1\card{T}^{\frac{3}{2}}}{1-2p},
		\end{align*}
		where $T^*$ is the underlying clustering of $T$ and $\bar{T}$ is the query result over $T \times T$.
	\end{theorem}
	We remark that Theorem~\ref{th app} is implicit in the proof of Theorem 1 in \cite{mathieu2010correlation}. 
	Here, we list the differences between the two results. 
	First, the goal of \cite{mathieu2010correlation} is to design a $(1+o_n(1))$-approximate algorithm, while here we focus on the additive error. Second, in \cite{mathieu2010correlation}, Mathieu and Schudy used the optimal solution to the SDP to do rounding.
	However, to the best of our knowledge, it is unknown if such solution can be obtained in polynomial time.
	This is why, in this paper, we consider a near optimal solution and we show that it is sufficient to achieve the same theoretical guarantee.
	Finally, in \cite{mathieu2010correlation}, the authors studied the performance of the algorithm in expectation, while here we give an exact bound for the probability that Algorithm~\ref{alg one round} succeeds. The proof of Theorem~\ref{th app} is given in to Appendix~\ref{apendix th5}.

	\subsection{Structure of the approximate correlation clustering}\label{sec oneround}
	
	In Section~\ref{sec SDP}, we have seen that, from Theorem~\ref{th app}, with high probability, the output of Algorithm~\ref{alg one round} is close to the underlying clustering function. 
	In this section, we study the structure of the output of Algorithm~\ref{alg one round}. 
	We will see that, if we run Algorithm~\ref{alg one round} over $T$ with a large enough size, then the clusters in the output must contain some big clusters and all such big clusters are distinct $\left(\eta,V^*_i\right)$-biased sets. 
	We summarize the main result of this section in the following theorem.

	\begin{theorem}\label{th one round}
		Let $T \subseteq V$ be a set of points such that $\card{T}=ts_t,$ where $s_t = \frac{c't^3 \log n}{\left(1-2p\right)^{6}}$, $c'$ is a large enough constant and $t \in \N^+$. Let $\tilde{T}= \textsc{ApproxCorrelationCluster}(T)$ and $\{\hat{T_1},\dots,\hat{T_h}\} : = \{\tilde{T_i} \in \tilde{T} \mid \card{\tilde{T_i}} >s_t/2\}$. Let $\eta = \frac{1}{4}+\frac{p}{2}.$ Under the semi-random model, with probability at least $1-1/\poly(n)$ the following events happen.
		\begin{itemize}
			\item If there is some $i \in [k]$ such that $\card{T_i^*}>s_t$, where $T^*_i=T \cap V^*_i$, then $h>0$.
			\item For every $i \in [h]$, $\hat{T_i}$ is an $\left(\eta,V^*_j\right)$-biased set for some $j \in [k]$. 
			\item For every $i,j \in [h], i \neq j$, there is no $\ell \in [k]$ such that $\hat{T_i},\hat{T_j}$ are both  $\left(\eta,V^*_\ell\right)$-biased sets.
		\end{itemize}
	\end{theorem}
	
	We sketch the proof of Theorem~\ref{th one round} here and leave the full proof to Appendix~\ref{apendix th6}. The key point of the proof of Theorem~\ref{th one round} is to show that, with high probability, any clustering function that is far from the underlying clustering function will have a large additive error. We summarize this result in Theorem~\ref{th tool} in Appendix~\ref{apendix th miss}. Based on this technical theorem, we will show that if any event in the statement of Theorem~\ref{th one round} does not happen, then $\tilde{T}$ must be significantly far from $T^*$. By Theorem~\ref{th app}, we know that, with high probability, $\tilde{T}$ has a small additive error and cannot be too far from $T^*$. Thus, the three events in the statement of Theorem~\ref{th one round} must happen together with high probability.

	\subsection{Sketch of the proof of Theorem~\ref{th final}}\label{sec proof}
	According to Theorem~\ref{th one round}, with probability at least $1-1/\poly(n)$, once the sampled set $T$ contains a subcluster of size $s_t$, $\textsc{ApproxCorrelationCluster}(T,1/\poly(n))$ will contain some large clusters $\hat{T_i}$. Each of them is an $(\eta,V^*_i)$-biased set and corresponds to a different underlying cluster. 
	Here, we choose $\eta=\frac{1}{4}+\frac{p}{2}$.
	By Hoeffding’s inequality, with probability at least $1-1/\poly(n)$, the corresponding subset $B_i$ of $\hat{T_i}$ is a $(\eta',V^*_i)$-biased set, where $\eta'=\frac{p+1}{3} \in (p,\eta)$. Lemma~\ref{lm test} then implies that we can use $B_i$ to recover $V^*_i$. 
	This shows that every element in the output is an underlying cluster. On the other hand, if there is an underlying cluster of size $\Omega\left(\frac{k^4\log n}{\left(1-2p\right)^{6}}\right)$ that has not been recovered, then at the end of the algorithm we have  $\card{V}=\Omega\left(\frac{k^4\log n}{\left(1-2p\right)^{6}}\right)$.
	However, after sampling $T$ from $V$ at most $O(\log k)$ times, $T$ contains a subcluster of size $s_t$ and the output will be updated. 
	This gives the correctness of Algorithm~\ref{alg clustering}.
	From the above argument, we can see that, throughout Algorithm~\ref{alg clustering}, we have $t= O(k)$, and thus $\card{T}=O(ks_k)$. 
	This also implies that we sample $T$ at most $O(k\log k)$ times. So the number of queries we spend on constructing sample sets is $O\left(k\log k \card{T}^2\right)$. From the correctness of the algorithm, we invoke Algorithm~\ref{alg test} a total of $O(nk)$ times to recover the underlying clusters. The total number of queries we perform to recover the underlying clusters is $O\left(\frac{nk\log n}{(1-2p)^2}\right)$. 
	Therefore, the query complexity of the algorithm can be bounded by $O\left(\frac{nk\log n}{(1-2p)^2}\right)+O\left(k\log k \card{T}^2\right)=O\left(\frac{nk\log n}{\left(1-2p\right)^2}+ \frac{k^9\log k \log^2 n}{\left(1-2p\right)^{12}} \right).$

	
	\section{An efficient parameter-free algorithm under the fully-random model}\label{sec application}
	In this section, we explain how Algorithm~\ref{alg clustering} can be used to design a parameter-free algorithm for the fully-random model.
	Since we have an efficient algorithm for the semi-random model, in principle, we can design an algorithm by guessing the parameter $(1-2p)$ from $0$ to $1$ and applying Algorithm~\ref{alg clustering} for each guess, until we are very close to the true parameter. However, in this way we will need to pay an extra $\log(\frac{1}{1-2p})$ factor for the query complexity and it will require us to test when to stop the guess. 
	To overcome these problems, we sample a constant number of points $A \subseteq V$, and query $A \times V$ before doing clustering. 
	We will see that, by counting the disagreements of point pairs in $A$, we can estimate a good upper bound $\bar{p}$ for the true error parameter $p$. 
	Then, we can run $\textsc{Clustering}(V,\bar{p})$ to recover large underlying clusters efficiently. We present the following algorithm whose performance is stated in Theorem~\ref{th fully} presented in the introduction and we leave the proof of Theorem~\ref{th fully} to Appendix~\ref{apendix th full}.
	
	\begin{algorithm}[H]
		\caption{\textsc{FClustering}$(V)$ (Parameter-free algorithm under fully-random model)}\label{alg fclustering}
		\begin{algorithmic}
			\State Randomly select $A \subseteq V$ such that $\card{A}=9$.
			\For{any pair of vertices $u,v \in A$}
			\State Set $\text{count}_{uv}$ to be the number of vertices $w \in V$ such that $\oracle(u,w) \neq \oracle(v,w)$.
			\EndFor
			\State Let $\bar{p}: = \frac{1}{2} -\frac{1}{4}\sqrt{1-\frac{2M}{n}}$ for $M:= \min \{\text{count}_{uv} \mid u,v \in A, u \neq v\}$.
			\State \Return \textsc{Clustering}$(V,\bar{p})$
			
		\end{algorithmic}
	\end{algorithm}

	
	\section{Acknowledgement}
	
	A. Del Pia is partially funded by ONR grant N00014-19-1-2322. Any opinions, findings, and conclusions or recommendations expressed in this material are those of the authors and do not necessarily reflect the views of the Office of Naval Research.

	\bibliographystyle{alpha}
	\bibliography{biblio}

	\appendix

\section{Proof of Lemma~\ref{lm test}} \label{apendix lm1}


For $u \in B,v \in V$ let $x_{uv}$ be the random variable such that 
\begin{align*}
	x_{uv} = \begin{cases}
		& 1 \ \text{if $(u,v)$ is not corrupted} \\
		& 0 \ \text{otherwise}.
	\end{cases}
\end{align*}

We first assume $v \in V^*_i$. It is sufficient to show $\sum_{u \in B\cap V^*_i} x_{uv}> \card{B}/2$ with high probability, because for every realization of $\oracle$, we have
\begin{align*}
	\sum_{u \in B}\oracle(u,v) = \sum_{u \in B\cap V^*_i} \oracle(u,v)+\sum_{u \in B\setminus V^*_i} \oracle(u,v)  \ge \sum_{u \in B\cap V^*_i} x_{uv}. 
\end{align*}
In expectation, we have 
\begin{align}\label{eq testin}
	\E \sum_{u \in B\cap V^*_i} x_{uv} = (1-p) \card{B \cap V^*_i} \ge (1-p)\left(\frac{1}{2}+\eta\right) \card{B} =\left( \frac{1}{2}-\frac{1}{2}p+\eta(1-p)\right ) \card{B}>\frac{1}{2} \card{B},
\end{align}
where the first inequality holds because $B$ is an $(\eta,V^*_i)$-biased set. The second inequality holds by the following calculation.
\begin{equation}\label{eq testin2}
	\begin{aligned}
		\left( \frac{1}{2}-\frac{1}{2}p+\eta(1-p)\right ) \card{B} - \frac{1}{2} \card{B} = \left(\eta-\eta p -\frac{1}{2}p\right)\card{B} & = \frac{1}{2}\left(p+\eta\right)\left(\frac{1}{2}-p\right)\card{B}+\left(\frac{3}{4}-\frac{1}{2}p\right)\left(\eta-p\right)\card{B} \\
		& \ge \frac{\eta}{2}\left(\frac{1}{2}-p\right)\card{B}>0,
	\end{aligned}
\end{equation}
where the first inequality holds because $0<p < \eta \le \frac{1}{2}$.

Since for every $u \in B \cap V^*_i$, $(u,v)$ is corrupted independently, by Hoeffding's inequality and \eqref{eq testin}, we have
\begin{align*}
	\Pr\left( \sum_{u \in B\cap V^*_i} x_{uv}<\frac{1}{2} \card{B}\right) & \le \exp \left(-2 \frac{\left(\E \sum_{u \in B\cap V^*_i} x_{uv} -\frac{1}{2} \card{B} \right)^2}{\card{B \cap V^*_i}}\right) \\
	& \le \exp \left(-\frac{1}{8}\eta^2(1-2p)^2\card{B}\right)
	\le \frac{1}{n^{10}}.
\end{align*}
Here, the second inequality holds by \eqref{eq testin2} and the last inequality follows by $\card{B}>\frac{80\log n}{\eta^2(1-2p)^2}.$ So every single point $v \in V^*_i$ has probability at most $1/n^{10}$ to be misclassified by \textsc{DegreeTest}$(v,B)$.

Next, we assume that $v \in V^*_j$ for some $j \neq i$. 
It is sufficient to show with high probability $\sum_{u \in B\setminus V^*_j}\left( 1-x_{uv}\right) \le  \eta \card{B}$, because 
\begin{align*}
	\sum_{u \in B}\oracle(u,v) = \sum_{u \in B\cap V^*_j}\oracle(u,v)+\sum_{u \in B\setminus V^*_j}\oracle
	(u,v) & \le \card{B \cap V^*_j} + \sum_{u \in B\setminus V^*_j} \left(1-x_{uv}\right) \\ 
	& \le \left(\frac{1}{2}-\eta\right)\card{B} + \sum_{u \in B\setminus V^*_j}\left( 1-x_{uv}\right).
\end{align*}
In expectation, we have 
\begin{align}\label{eq testoff}
	\E \sum_{u \in B\setminus V^*_j} \left(1-x_{uv}\right) = p \card{B\setminus V^*_j}<\eta \card{B}.
\end{align}
Since for every $u \in B \cap V^*_j$, $(u,v)$ is corrupted independently, by Hoeffding's inequality and \eqref{eq testoff}, we have 
\begin{align*}
	\Pr\left(\sum_{u \in B\setminus V^*_j} \left(1-x_{uv}\right)>\eta\card{B}\right) & \le \exp\left(-2 \frac{\left(\eta\card{B} - p\card{B \setminus V^*_j}\right)^2}{\card{B \setminus V^*_j}}\right) \\
	& \le \exp\left(-2\left(\eta-p\right)^2\card{B}\right) \le \frac{1}{n^{10}}.
\end{align*}
Thus, for every $v \in V$, with probability at most $1/n^{10}$, $v$ will be misclassified by \textsc{DegreeTest}$(v,B)$. By union bound, we know \textsc{DegreeTest}$(v,B)$ correctly classifies every $v \in V$ with probability at least $1-1/n^9$.
%
$\hfill \square$

\section{Technical discussion of efficient algorithm}\label{apendix sec3}
In this section, we give a discussion of previous techniques for designing efficient algorithms under the noise model. We will take a disagreement counting method as an example and show how it can be applied to design algorithms under the semi-random model and where its limitation is. We will consider a slightly weaker model here.

\begin{definition}(Non-adaptive semi-random model of clustering with the faulty oracle)
	Under definition~\ref{model 1}, the oracle $\oracle$ is defined in the following way. There is a set of unknown parameter $\{p_{uv} \ge 0 \mid u,v \in V\}$ and
	a known error parameter $p \in (0,1/2)$ such that, for every point pair $u,v$, with probability $1-p_{uv}$, $\oracle(u,v)$ outputs the true label of $(u,v)$ and with probability $p_{uv}$, $\oracle(u,v)$ outputs the wrong label of $(u,v)$, where $0 \le p_{uv} \le p$. 
\end{definition}

\subsection{A disagreement counting method to obtain $(\eta,V^*_i)$-biased sets}	\label{apendix bicluster}

From Lemma~\ref{lm test}, we know if we get an $(\eta,V^*_i)$-biased set of size $O(\log n)$, we can use it to recover $V^*_i$, by making $O(n\log n)$ queries. Thus, the key technique for designing an efficient algorithm is to obtain such $(\eta,V^*_i)$-biased sets by making a small number of queries. 
To address this problem, 
we start with a simple disagreement counting method, which has been heavily used under the fully-random model \cite{bansal2004correlation,mazumdar2017clustering,green2020clustering}. We consider the following simple procedure. Lemma~\ref{lm check} gives the theoretical guarantee for this simple procedure.

\begin{algorithm}[H]
	\caption{\textsc{DisagreementTest}$(u,v,T)$ (Check if $u,v$ are in the same cluster via a set $T \subseteq V$)}\label{alg check}
	\begin{algorithmic}
		\If{$\text{count}_v=\card{\{w \in T \mid \oracle(u,w) \neq \oracle
				(v,w)\}}>\frac{\card{T}}{2}$} \Return ``No" \textbf{else} \Return ``Yes''
		\EndIf
	\end{algorithmic}
\end{algorithm}

\begin{lemma}\label{lm check}
	Under the non-adaptive semi-random model, suppose $k=2$. If $\card{T} \ge \frac{100 \log n}{\left(1-2p\right)^4} $, then for every point pair $(u,v)$, with probability at least $1-1/\poly(n)$, the following event happens.
	\begin{itemize}
		\item If $u,v \in V^*_i$ for some $i \in [2]$, \textsc{DisagreementTest}$(u,v,T)$ returns ``Yes".
		\item If $u \in V^*_i, v \in V^*_j$ for $i \neq j$, \textsc{DisagreementTest}$(u,v,T)$ returns ``No".
	\end{itemize}
\end{lemma}


\begin{proof}
	We first assume $u,v \in V^*_i$ for some $i \in [2]$. We have for every point pair $(u,v)$ and for every $w \in V$,
	\begin{align*}
		\Pr\left(\oracle(u,w) \neq \oracle
		(v,w) \right) = p_{uw}(1-p_{vw})+p_{vw}(1-p_{uw}) \le 2p(1-p) = \frac{1}{2} - \frac{1}{2}(1-2p)^2,
	\end{align*}
	since $\oracle(u,w) \neq \oracle
	(v,w)$ happens if and only if $\oracle$ gives a wrong answer to exactly one of $(u,w)$ and $(v,w)$.
	So when $u,v$ in the same cluster $V^*_i$, in expectation, we have 
	\begin{align*}
		\E \text{count}_v \le \frac{\card{T}}{2} - \frac{1}{2}(1-2p)^2 \card{T}.
	\end{align*}
	By Hoeffding's inequality, we have 
	\begin{align*}
		\Pr \left(\text{count}_v > \frac{\card{T}}{2} \right) \le \exp\left(- \frac{\left(1-2p\right)^4\card{T}}{2}\right) \le \frac{1}{n^{50}}.
	\end{align*}
	Next, we assume that $u,v$ belong to different clusters. For every such point pair $(u,v)$ and for every $w \in V$, we have 
	\begin{align*}
		\Pr\left(\oracle(u,w) \neq \oracle
		(v,w) \right) = p_{uw}p_{vw}+\left(1-p_{uw}\right)\left(1-p_{vw}\right) \ge 1-2p(1-p) = \frac{1}{2}+ \frac{1}{2}\left(1-2p\right)^2.
	\end{align*}
	By Hoeffding's inequality, we have 
	\begin{align*}
		\Pr \left(\text{count}_v \le \frac{\card{T}}{2} \right) \le \exp\left(- \frac{\left(1-2p\right)^4\card{T}}{2}\right) \le \frac{1}{n^{50}}.
	\end{align*}
\end{proof}

Based on Lemma~\ref{lm check}, we get the following simple algorithm for the special case where $k=2$ under the non-adaptive semi-random model.

\begin{algorithm}[H]
	\caption{\textsc{Biclustering}$(V,p)$ (Exactly recover 2 underlying clusters)}\label{alg k=2}
	\begin{algorithmic}
		
		\State Randomly select a point $u \in V$
		
		\State Let $T_1=\{u\}, T_2=\emptyset$
		
		\For{$i \in [\frac{640 \log n}{\left(1-2p\right)^2}]$}
		\State Randomly select a subset $T \subseteq V$ of size $\frac{100\log n}{\left(1-2p\right)^4}$ 
		\State Select $v \in V \setminus \left(T_1 \cup T_2\right)$
		\If{\textsc{DisagreementTest}$(u,v,T)$=``Yes''}
		\State $T_1 \gets T_1 \cup \{v\}$
		\Else
		\State $T_2 \gets T_2 \cup \{v\}$
		\EndIf
		
		\EndFor
		
		\State Let $B = \text{argmax}\{\card{T_1},\card{T_2}\}$
		
		\State Let $\tilde{V}_1=\{v \in V \mid \textsc{DegreeTest}(v,T_1)=\text{``Yes''}\}$ $\tilde{V}_2=\{v \in V \mid \textsc{DegreeTest}(v,T_1)=\text{``No''}\}$
		
		\State \Return $\tilde{V}_1,\tilde{V}_2$
	\end{algorithmic}
\end{algorithm}

\begin{theorem}\label{th k=2}
	There is an algorithm \textsc{Biclustering}$(V,p)$, such that under the non-adaptive semi-random model, suppose $k=2$, with probability at least $1-1/\poly(n)$, \textsc{Biclustering}$(V,p)$ exactly recovers $V^*$. Furthermore, the query complexity of \textsc{Biclustering}$(V,p)$ is $O\left(\frac{n\log n}{(1-2p)^2}+\frac{\log^2n}{(1-2p)^6}\right)$ and the running time of \textsc{Biclustering}$(V,p)$ is $O\left(\frac{n\log n}{(1-2p)^2}+\frac{\log^2n}{(1-2p)^6}\right)$.
\end{theorem}

\subsection{Proof of Theorem~\ref{th k=2}}\label{apendix th1}


According to Lemma~\ref{lm check} and union bound, we know that with probability at least $1-1/\poly(n)$, every point $v \in T_1$ belongs to the same underlying cluster as $u$ and every point $v \in T_2$ belongs to the different underlying cluster from $u$. Thus, we know that $B \subseteq V^*_i$ for some $i \in [2]$. Furthermore, since $\card{B} \ge \frac{320\log n}{(1-2p)^2} $, according to Lemma~\ref{lm test}, with $\eta = 1/2$, we know that with probability at least $1-1/\poly(n)$, all points in $\tilde{V_i}$, for $i \in [2]$, come from the same underlying cluster. Combine the above argument together, with probability at least $1-1/\poly(n)$, \textsc{Biclustering}$(V,p)$ exactly recover $V^*$. 

It is simple to check the running time and the query complexity of \textsc{Biclustering}$(V,p)$ are the same.
Since in \textsc{Biclustering}$(V,p)$, we invoke \textsc{DisagreementTest}$(u,v,T)$ at most $ O \left(\frac{640 \log n}{\left(1-2p\right)^2}\right) $ times, and each time \textsc{DisagreementTest}$(u,v,T)$ queries $ O\left(\frac{\log n}{\left(1-2p\right)^4}\right) $ times. The total number of queries is $O\left( \frac{\log^2n}{(1-2p)^6}\right)$ in this stage. Since $\card{B} \le \frac{640 \log n}{\left(1-2p\right)^2} $ and we invoke \textsc{DegreeTest} $n$ times, in this stage the total number of queries is $O\left( \frac{n\log n}{(1-2p)^2}\right)$. Thus, the query complexity of \textsc{Biclustering}$(V,p)$ is $O\left(\frac{n\log n}{(1-2p)^2}+\frac{\log^2n}{(1-2p)^6}\right)$. 
$\hfill \square$

\bigskip

We can see that our algorithm achieves the same query complexity as the algorithm in \cite{green2020clustering} does, but our algorithm can succeed in a stronger model. It seems that using similar ideas we can design an efficient algorithm for more general settings. Unfortunately, as it turns out, this is a wrong approach. Although the naive disagreement counting method works very well under the fully-random model, under the semi-random model it only works on some very restrictive cases. 


\subsection{Disagreement counting method fails under general semi-random model}

In fact, we can construct examples to show the disagreement counting method can easily fail, even if we have a small number of clusters and a constant level of noise rate. In the first example, we will show in Lemma~\ref{lm check}, the assumption of the non-adaptive model is necessary. If we run Algorithm~\ref{alg check} under the adaptive semi-random model, we will fail even if we have only two clusters.

\begin{example}\label{ex general}
	Under the adaptive semi-random model, suppose $k=2$ and $p \ge 1- \frac{\sqrt{2}}{2}$. There is an instance such that for every subset $T \subseteq V$ and for every point pair $u,v$, with probability $1/2$, the following event happens
	\begin{itemize}
		\item If $u,v \in V^*_i$ for some $i \in [2]$, \textsc{DisagreementTest}$(u,v,T)$ returns ``No".
		\item If $u \in V^*_i, v \in V^*_j$ for $i \neq j$, \textsc{DisagreementTest}$(u,v,T)$ returns ``Yes".
	\end{itemize}
\end{example}


\begin{proof}
	It is sufficient to show that for every $q \in [(1-p)^2,1-p+p^2]$ and $r \in [p(1-p),p(2-p)]$, there is an adversary such that for every $u,v$ in the same underlying cluster, $\Pr\left(\oracle(u,w)=\oracle(v,w)\right) = q, \forall w \in V$ and for every $u,v$ in different underlying clusters, $\Pr\left(\oracle(u,w) = \oracle(v,w)\right) = r, \forall w \in V.$ This is because when $p \ge 1- \frac{\sqrt{2}}{2}$, we have 
	\begin{align*}
		p(1-p) \le (1-p)^2 \le \frac{1}{2} \le p(2-p) \le 1-p+p^2.
	\end{align*}
	By setting $q=r=1/2$, we know for every tuple $u,v,w$, $\Pr\left(\oracle(u,w) = \oracle(v,w)\right) = 1/2$, which directly implies that 
	\begin{align*}
		\Pr\left(\text{count}_v > \frac{\card{T}}{2}\right)
		&= \Pr\left(\textsc{DisagreementTest}(u,v,T)  =\text{``No"}\right)\\
		&=  \Pr\left(\textsc{DisagreementTest}(u,v,T)=\text{``Yes"}\right)= \Pr\left(\text{count}_v \le  \frac{\card{T}}{2}\right)
		= \frac{1}{2},
	\end{align*}
	for every choice of $u,v,T$. This implies the adversary can make it impossible to distinguish if two point $u,v$ in the same underlying cluster or not by counting the disagreement. Now we show this.
	
	We consider an adversary that works in the following way. For every point pair $u,v$ in a same underlying cluster and any point $w \in V$, let $(u,w)$ be the first point pair we query. The adversary always chooses to output a wrong label of $(u,w)$, when $(u,w)$ is corrupted. If $\oracle(u,w)$ outputs the correct label, the adversary chooses to output a correct label of $(v,w)$ with probability $p_{cc} \in [0,1]$, when $(v,w)$ is corrupted. If $\oracle(u,w)$ outputs the wrong label, the adversary chooses to output a correct label of $(v,w)$ with probability $p_{nc} \in [0,1]$, when $(v,w)$ is corrupted. Thus, 
	\begin{align*}
		\Pr\left(\oracle(u,w) = \oracle(v,w) \right) = (1-p)(1-p+pp_{cc})+p^2(1-p_{nc}) \in [(1-p)^2,1-p+p^2],
	\end{align*}
	where both the upper bound and the lower bound are achievable. By continuity, for every $q \in [(1-p)^2,1-p+p^2]$, we can select proper parameters $p_{cc},p_{nc}$ to exactly match the probability. Similarly, for every $(u,v)$ in different clusters, let $(u,w)$ be the first point pair we query. The adversary always chooses to output a wrong label of $(u,w)$, when $(u,w)$ is corrupted. If $\oracle(u,w)$ outputs the correct label, the adversary chooses to output a correct label of $(v,w)$ with probability $p'_{cc} \in [0,1]$, when $(v,w)$ is corrupted. If $\oracle(u,w)$ outputs the wrong label, the adversary chooses to output a correct label of $(v,w)$ with probability $p'_{nc} \in [0,1]$, when $(v,w)$ is corrupted. Thus, 
	\begin{align*}
		\Pr\left(\oracle(u,w) = \oracle(v,w) \right) = (1-p)p(1-p'_{cc})+p(1-p+pp'_{nc}) \in [p(1-p),p(2-p)],
	\end{align*}
	where both the upper bound and the lower bound are achievable. By continuity, for every $r \in [p(1-p),p(2-p)]$, we can select proper parameter $p'_{cc},p'_{nc}$ to exactly match the probability.
\end{proof}


We have seen Algorithm~\ref{alg check} can easily fail under the adaptive semi-random model, because of the power of the adversary. However, even if we work on the non-adaptive semi-random model, it is still hard to apply the disagreement counting technique to cases where we have more clusters.

\begin{example}\label{ex nonadaptive}
	Under the non-adaptive semi-random model, there is an instance with $k=3$, $\card{V^*_i} = n/3$, for $i \in [3]$ and $p=1/3$ such that for a subset $T$ uniformly selected from $V$ and for every $u,v$ that are in different underlying clusters, with probability $1/2$, \textsc{DisagreementTest}$(u,v,T)$ returns ``Yes".
\end{example}


\begin{proof}
	We design the oracle in the following way. $\oracle(u,v) = 0$ with probability $p_{in}$ for every $u,v$ in the same underlying cluster. $\oracle(u,v) = 1$ with probability $p_{out}$ for every $u,v$ in different underlying clusters. Let $(u,v)$ be a pair of points in different underlying clusters. Denote by $V^*_u$ the underlying cluster that $u$ belongs to and denote by $V^*_v$ the underlying cluster that $v$ belongs to. Let $w \in V^*_u \cup V^*_v$, Then 
	\begin{align*}
		\Pr\left(\oracle(u,w) = \oracle(v,w) \right) = \left(1-p_{in}\right)p_{out} + \left(1-p_{out}\right)p_{in}.
	\end{align*}
	Let $w \in V \setminus \left(V^*_u \cup V^*_v\right)$, then
	\begin{align*}
		\Pr\left(\oracle(u,w) = \oracle(v,w) \right) = p^2_{out} + \left(1-p_{out}\right)^2.
	\end{align*}
	We pick $p_{in} = \frac{1}{2}-\frac{1}{4\sqrt{2}}$ and $p_{out} = \frac{1}{2}-\frac{\sqrt{2}}{4}$. It can be checked that $p_{out}<p_{in}<1/3$. If we uniform pick some $w \in V$, we have 
	\begin{align*}
		\Pr\left(\oracle(u,w) = \oracle(v,w) \right) = \frac{2}{3} \left( \left(1-p_{in}\right)p_{out} + \left(1-p_{out}\right)p_{in} \right) + \frac{1}{3} \left(  p^2_{out} + \left(1-p_{out}\right)^2 \right) = \frac{1}{2}.
	\end{align*}
	This implies for every subset $T$ uniformly selected from $V$ and every $u,v$ in different underlying clusters,
	\begin{align*}
		\Pr\left(\textsc{DisagreementTest}(u,v,T)=\text{``Yes"}\right) = \Pr\left(\text{count}_v \le \frac{\card{T}}{2}\right) = \frac{1}{2}.
	\end{align*}
\end{proof}

\section{Missing proofs in Section~\ref{sec information}}

\subsection{Proof of Theorem~\ref{th estimation semi}}	\label{apendix information}
We first show the correctness of Algorithm~\ref{alg information}. Let $\textsc{Estimation}(V,p)=\{\tilde{V_1},\dots,\tilde{V_\ell}\}$. We start by showing every $\tilde{V_i}$ is an underlying cluster. Every $\tilde{V_i}$ is constructed by adding points from $V$ to some $T' \in \textsc{FindBigClusters}(T)$, for some $T$ through the algorithm. According to Theorem~\ref{th Process}, $T' = T \cap V^*_i$ for some $i \in [k]$ and $\card{T'} \ge \frac{320\log n}{(1-2p)^2}$. Thus, according to Lemma~\ref{lm test}, a point $v$ is added to $T'$ if and only if $v \in V^*_i$, which implies $\tilde{V_i} \subseteq V^*_i$. On the other hand, every $v \in V^*_i$ must be added to $\tilde{V_i}$. This is because at the time the first $T'=T \cap V^*_i$ is found by \textsc{FindBigClusters}$(T)$, $v$ is in either $T$ or $V$. If $v \in T$, then $v \in T'$, otherwise, $v$ will be added to $T'$ according to Lemma~\ref{lm test}. So every element in the output of Algorithm~\ref{alg information} is an underlying cluster. Next, if there is an underlying cluster $V^*_i$ of size at least $\frac{321 \log n}{(1-2p)^2}$ that is not recovered, then $V^*_i \subseteq T$ at the end of the algorithm. However, in this case, according to Theorem~\ref{th Process}, with probability at least $1-1/\poly(n)$, $\textsc{FindBigClusters}(T) \neq \emptyset$ and the output will be updated. Thus, Algorithm~\ref{alg information} recovers all underlying clusters of size at least $\Omega\left(\frac{\log n}{(1-2p)^2}\right)$ with probability at least $1-1/\poly(n)$.

Finally, we show the query complexity of Algorithm~\ref{alg information}. Given a point $v$, before we put $v$ into $T$, we check if we can
assign $v$ to some cluster in $C$ and suppose $\card{C} = k'$. To check if $v$ can be assigned to a cluster in $C$, we need to query at most $O\left(\frac{k'\log n}{(1-2p)^2}\right)$ times. If $v$ can be assigned to a cluster in $C$, we do not need to query $v$ with other points anymore.
If $v$ cannot be assigned to a cluster in $C$, we add $v$ to $T$ and query $O(T)$ times. We notice that $\card{T} = O\left(\frac{(k-k')\log n}{(1-2p)^2}\right)$, because otherwise $T$ will contain a subcluster of size $\Omega\left(\frac{\log n}{(1-2p)^2}\right)$ and $T$ will be updated. This implies to assign a given point $v$, we need to query
\begin{align*}
	O\left(\frac{k'\log n}{(1-2p)^2}\right) + O\left(\frac{(k-k')\log n}{(1-2p)^2}\right) = O\left(\frac{k\log n}{(1-2p)^2}\right)
\end{align*}
times. Thus, the query complexity of Algorithm~\ref{alg information} is $O\left(\frac{nk\log n}{(1-2p)^2}\right)$.
$\hfill \square$

\subsection{Proof of Theorem~\ref{th Process}} \label{apendix th2}

In this part, we give the proof of Theorem~\ref{th Process}. We first state two technical lemmas and use these two lemmas to prove Theorem~\ref{th Process}. Then we give a full proof for the two technical lemmas.

The key part for the proof of Theorem~\ref{th Process} is to show the following two technical lemmas. Intuitively, we want to show if $T \subseteq V$ contains a large subcluster, then with high probability, this subcluster has no negative cut, while on the other hand, with high probability, any large subset of $T$ that is not a subcluster must have a negative cut. 

\begin{lemma}\label{lm clique}
	Under the semi-random model, let $S \subseteq V^*_i$ for some $i \in [k]$ such that $\card{S} \ge \frac{320 \log n}{(1-2p)^2}$, then with probability at least $1-1/\poly(n)$, $val_S = \min_{A \subseteq S} \sum_{u \in A}\sum_{v \in S\setminus A}w_{uv}>0$.
\end{lemma}

\begin{lemma}\label{lm cut}
	Let $T \subseteq V$ such that $\card{T} \ge \frac{320 \log n}{(1-2p)^2}$. For $i \in [k]$, let $T_i:= T \cap V^*_i$. Denote by $t: = \max_{i \in [k]} \card{T_i}$ and $s: = \frac{320 \log n}{(1-2p)^2}$. Under the semi-random model, with probability at least $1-1/\poly(n)$, any subset $S \subseteq T$ such that $\card{S} \ge \max \{t,s\}$ and $S$ is not a subcluster satisfies $val_S = \min_{A \subseteq S} \sum_{u \in A}\sum_{v \in S\setminus A}w_{uv} \le 0$.
\end{lemma}

Now we prove Theorem~\ref{th Process} using Lemma~\ref{lm clique} and Lemma~\ref{lm cut}.

\begin{proof} \textbf{of Theorem~\ref{th Process}} \ 
	For $i \in [k]$, let $T_i:= T \cap V^*_i$. Denote by $t: = \max_{i \in [k]} \card{T_i}$ and $s: = \frac{320 \log n}{(1-2p)^2}$. If $t < s$, according to Lemma~\ref{lm cut}, with probability at least $1-1/\poly(n)$, any subset $S$ of $T$, such that $\card{S} \ge s$ must have $val_S \le 0$. In this case, no element will be added to the output. If $t \ge s$, let $T'$ be the largest subcluster contained in $T$. According to Lemma~\ref{lm clique}, $val_{T'}>0$. Furthermore, according to Lemma~\ref{lm cut}, any subset of $T$ that has a larger size than $T'$ will contain a negative cut. Thus, the largest subcluster $T'$ is the largest subset of $T$ that contains no negative cut and will be added to the output. Since there are $k \le n$ underlying clusters, by union bound, we know with probability at least $1-1/\poly(n)$, the output of Algorithm~\ref{alg Process} is $\{T \cap V^*_i \mid \card{T \cap V^*_i}  \ge \frac{320 \log n}{(1-2p)^2}\}$.
\end{proof}

\subsubsection{Proof of Lemma~\ref{lm clique}}

Let $(A,S\setminus A)$ be a partition of $S$ such that $\card{A} \le \card{S}/2$.
For $u \in A,v \in S \setminus A$, let $x_{uv}$ be the random variable such that 
\begin{align*}
	x_{uv} = \begin{cases}
		& 1 \ \text{if $(u,v)$ is not corrupted} \\
		& -1 \ \text{otherwise}.
	\end{cases}
\end{align*}
Clearly, for every realization of the query result,
\begin{align*}
	\sum_{u \in A}\sum_{v \in S\setminus A}w_{uv} \ge \sum_{u \in A}\sum_{v \in S\setminus A}x_{uv}.
\end{align*}
Since 
\begin{align*}
	\E \sum_{u \in A}\sum_{v \in S\setminus A}x_{uv} =  \left(1-2p\right)\card{A}\card{S\setminus A},
\end{align*}
by Hoeffding's inequality, we have 
\begin{align*}
	\Pr\left(\sum_{u \in A}\sum_{v \in S\setminus A}w_{uv} \le 0\right) \le \Pr\left(\sum_{u \in A}\sum_{v \in S\setminus A}x_{uv} \le 0\right) \le  \exp \left( -\frac{\left(1-2p\right)^2\left(\card{A}\card{S\setminus A}\right)}{2} \right).
\end{align*}
Suppose $\card{A} = t$, where $t \in [\card{S}/2]$. By union bound, we have 
\begin{align*}
	\Pr \left(\exists A \subseteq S, \card{A} = t, \sum_{u \in A}\sum_{v \in S\setminus A}w_{uv} \le 0 \right) & \le \tbinom{\card{S}}{t}\exp \left( -\frac{\left(1-2p\right)^2\left(\card{A}\card{S\setminus A}\right)}{2} \right) \\
	& \le \exp \left( t\log \card{S}- \frac{\left(1-2p\right)^2t(\card{S}-t)}{2}\right) \\
	& \le \exp \left(-t\left(\frac{(1-2p)^2\card{S}}{4} - \log \card{S}\right)\right) \\ 
	& \le \exp \left(-24\log n\right) = \frac{1}{n^{24}},
\end{align*}
where the third inequality follows by $t \le \card{S}/2$ and the last inequality holds because $\frac{320 \log n}{(1-2p)^2} \le \card{S} \le n$ and $t \ge 1$. Using union bound again, we can get 
\begin{align*}
	\Pr \left(val_S \le 0\right) \le \sum_{t=1}^{\card{S}/2}\Pr \left(\exists A \subseteq S, \card{A} = t, \sum_{u \in A}\sum_{v \in S\setminus A}w_{uv} \le 0 \right) \le \frac{1}{\poly(n)}.
\end{align*}
So with probability at least $1-1/\poly(n)$, $val_S>0$.
$\hfill \square$	

\subsubsection{Proof of Lemma~\ref{lm cut}}


Let $S \subseteq T$. Denote by $m=\card{S}  \ge \max \{t,s\}$. Denote by $S^*$ the largest subcluster contained in $S$.
Let $(A,S\setminus A)$ be a partition of $S$ such that $\card{A} \le \card{S}/2$.
For $u \in A,v \in S \setminus A$, let $x_{uv}$ be the random variable such that 
\begin{align*}
	x_{uv} = \begin{cases}
		& 1 \ \text{if $(u,v)$ is not corrupted} \\
		& -1 \ \text{otherwise}.
	\end{cases}
\end{align*}

We first show with probability at least $1-1/\poly(n)$, any subset $S$ that satisfies the statement of Lemma~\ref{lm cut} and $\card{S^*} \le m/4$, must have $val_S \le 0$. We first fix such a set $S$.
We show with high probability we can construct a subset $A \subseteq S$ such that $\sum_{u \in A}\sum_{v \in S\setminus A}w_{uv} \le 0$. Without loss of generality, we assume that $0 \le \card{S \cap V^*_1} \le \dots \le \card{S \cap V^*_k} \le \card{S}/4$. 

We define $A: = \cup_{j=1}^{i^*} \left(S \cap V^*_j\right)$, where $i^*$ is the largest index such that $\card{\cup_{j=1}^{i^*} \left(S \cap V^*_j\right)} \le \card{S}/2$. We can see $\card{A}>\card{S}/4$, otherwise, we can put $S \cap V^*_{i^*+1}$ into $A$ and keep $\card{A} \le \card{S}/2$. This implies $\card{A}\card{S\setminus A} \ge \card{S}^2/8.$ 
In expectation, we have 
\begin{align*}
	\E  \sum_{u \in A}\sum_{v \in S\setminus A}-x_{uv} =  -\left(1-2p\right)\card{A}\card{S\setminus A}.
\end{align*}
By Hoeffding's inequality, we have 
\begin{align*}
	\Pr\left(val_S \ge 0\right)  &\le \Pr \left( \sum_{u \in A}\sum_{v \in S\setminus A}w_{uv} \ge 0 \right)  
	\le \Pr \left( \sum_{u \in A}\sum_{v \in S\setminus A}-x_{uv} \ge 0 \right) \\
	&	\le \exp\left(-\frac{(1-2p)^2\card{A}\card{S \setminus A}}{2}\right) 
	\le \exp \left(-\frac{(1-2p)^2m^2}{16}\right).
\end{align*}

By union bound, we have 
\begin{align*}
	\Pr \left(\exists S, \card{S}=m, \card{S^*} \le m/4, val_S>0\right) & \le \tbinom{n}{m}  \exp \left(-\frac{(1-2p)^2m^2}{16}\right) \\
	& \le \exp \left(m \log n - \frac{(1-2p)^2m^2}{16}\right) \\
	& =  \exp \left(-m \left( \frac{(1-2p)^2m}{16} - \log n\right)\right) \\
	& \le \exp (-4m\log n) \le \frac{1}{n^{80}},
\end{align*}
where in the third inequality, we use the fact $\frac{(1-2p)^2m}{16} \ge \frac{(1-2p)^2s}{16}  \ge 5\log n$.

Again, using union bound over $m$, we get 
\begin{align*}
	\Pr \left(\exists S, \card{S} \ge \max\{t,s\}, \card{S^*} \le \card{S}/4, val_S>0\right) \le n \Pr \left(\exists S, \card{S}=m, \card{S^*} \le m/4, val_S>0\right) \le \frac{1}{\poly(n)}.
\end{align*}

Thus, with probability at least $1-1/\poly(n)$,  any subset $S$ that satisfies the statement of Lemma~\ref{lm cut} and $\card{S^*} \le m/4$ must have $val_S \le 0$.

Next, we show with probability at least $1-1/\poly(n)$, any subset $S$ that satisfies the statement of Lemma~\ref{lm cut} and $m/4 \le \card{S^*} \le  m-1$ must have $val_S \le 0$. To simplify the notation, we denote by $r:=\card{S \setminus S^*}$. We fix such a set $S$. We show with high probability $\sum_{u \in S^*}\sum_{v \in S\setminus S^*}w_{uv} \le 0$.  By Hoeffding's inequality, we have 
\begin{align*}
	\Pr\left(val_S \ge 0\right)  &\le \Pr \left( \sum_{u \in S^*}\sum_{v \in S\setminus S^*}w_{uv} \ge 0 \right)  
	\le \Pr \left( \sum_{u \in S^*}\sum_{v \in S\setminus S^*}-x_{uv} \ge 0 \right) \\
	&	\le \exp\left(-\frac{(1-2p)^2\card{S^*}\card{S \setminus S^*}}{2}\right) 
	=\exp \left(-\frac{(1-2p)^2(m-r)r}{16}\right).
\end{align*}
Then we upper bound the number of such $S$. We first choose an index $i \in [k]$ such that $S \cap V^*_i = S^*$. The number of choices of the index is at most $k$. Then we choose $m-r$ points from $T \cap V^*_i$. The number of choices of the points is at most $\tbinom{t}{m-r}$. Finally, we choose $r$ points from the rest points in $T$. The number of such choices is at most $\tbinom{\card{T}}{r}$. So the number of such $S$ is at most
\begin{align*}
	k\tbinom{t}{m-r}\tbinom{\card{T}}{r} \le k\tbinom{m}{m-r}\tbinom{n}{r} = k\tbinom{m}{r}\tbinom{n}{r} \le k\exp \left(r (\log m+\log n)\right) \le n \exp(2r\log n),
\end{align*}
where the first inequality follows by $t \le m$ and $\card{T} \le n$ and the last inequality follows by $k \le n$ and $m \le n$. By union bound, we have 
\begin{align*}
	\Pr \left(\exists S, \card{S}=m, \card{S^*} = m-r, val_S>0\right) & \le n\exp \left(2r\log n-\frac{(1-2p)^2(m-r)r}{16}\right) \\
	& = n \exp \left(-2r\left(\frac{(1-2p)^2(m-r)}{32} -\log n \right)\right) \\
	& \le n \exp \left(-2r\left(\frac{3(1-2p)^2m}{128} -\log n \right)\right) \\
	& \le n \exp \left(-10r\log n\right) \le \frac{1}{n^{9}}.
\end{align*}
Here, the second inequality follows by $1 \le r \le m/4$. The third inequality holds because $m \ge s$ and the last inequality holds since $r \ge 1$. 
Since $1 \le r \le m \le n$, by applying union bound over $r$ and $m$, we have 
\begin{align*}
	\Pr \left(\exists S, \card{S} \ge \max\{t,s\}, \card{S}/4 < \card{S^*} < \card{S}, val_S>0\right) \le n^2 \Pr \left(\exists S, \card{S}=m, \card{S^*} = m-r, val_S>0\right) \le \frac{1}{\poly(n)}.
\end{align*}

Thus, with probability at least $1-1/\poly(n)$, any subset $S \subseteq T$ such that $\card{S} \ge \max \{t,s\}$ and $S$ is not a subcluster satisfies $val_S = \min_{A \subseteq S} \sum_{u \in A}\sum_{v \in S\setminus A}w_{uv} \le 0$.
$\hfill \square$	


\subsection{Example where Algorithm~1 in \cite{mazumdar2017clustering} fails} \label{apendix ex3}

In this part, we present an example where Algorithm 1 in \cite{mazumdar2017clustering} fails. We remark that the main difference of Algorithm~\ref{alg information} in this paper and Algorithm 1 in \cite{mazumdar2017clustering} is that given a sampled set $T$, we find the largest subset of $T$ that contains no negative cut, while they compute the heaviest subgraph of $T$.

\begin{example}\label{ex semi}
	Under the semi-random model, there is an instance with $k=2$ and $p \ge \frac{2}{5}$ such that with probability at least $1-o_n(1)$, Algorithm 1 in \cite{mazumdar2017clustering} fails to recover any cluster.
\end{example}

\begin{proof}
	Consider $V = V^*_1 \cup V^*_2$, where $\card{V^*_1}=\card{V^*_2}=\frac{n}{2}$. Let $p=2/5$ be the error parameter. We run Algorithm 1 in \cite{mazumdar2017clustering} on this example. The first step of the algorithm is to sample a set $T$ of size $s=\frac{16\log n}{(1-2p)^2}$. By Chernoff bound, with probability at least $1-o_n(1)$, $\card{T \cap V^*_1} \ge \frac{3s}{7}$ and $\card{T \cap V^*_1} \ge \frac{3s}{7}$. After getting such a sampled set $T$, the adversary works in the following way. For every corrupted point pair $(u,v)$, the adversary outputs the true label if $u,v$ are in the same underlying clusters and otherwise outputs a wrong label. 
	
	We will next show with probability at least $1-o_n(1)$, the largest subgraph of $T$ is $T$. To simplify the notation, let $A = T \cap V^*_1$ and $B= T \cap V^*_2$. We know for every $(u,v)$ such that $u \in A$ and $v \in B$, $w_{uv}=-1$ with probability $p$. By Hoeffding's inequality and union bound, with probability at least $1-o_n(1)$, for every $u \in A$, we have $\sum_{v \in B}w_{uv} \ge -\frac{\card{B}}{3}$ and for every $v \in B$, we have $\sum_{u \in A}w_{uv} \ge -\frac{\card{A}}{3}$. Now, let $S$ be an arbitrary subset of $T$ and assume of $V^*_i$ is the underlying cluster that the majority of points of $S$ come from. It is not hard to see, adding another point $v \in T \cap V^*_i$ will not decrease the total weight of $S$. So we can without loss of generality assume $A \subseteq S$ or $B \subseteq S$. We deal with the case when $A \subseteq S$ and the proof is the same when $B \subseteq S$. Suppose $\card{S \cap B} = x\card{B}$, where $x \in [0,1]$. Now we add the rest $(1-x)\card{B}$ points from $B$ to $S$ and we get $T$. Since for every $v \in B$, we have $\sum_{u \in A}w_{uv} \ge -\frac{\card{A}}{3}$, we know the increment of weight is at least
	\begin{align*}
		x(1-x)\card{B}^2+\frac{(1-x)^2\card{B}^2}{2}-\frac{(1-x)\card{A}\card{B}}{3} \ge x(1-x)\card{B}^2+\frac{(1-x)^2\card{B}^2}{2}-\frac{4(1-x)\card{B}^2}{9}>0,
	\end{align*}
	where the first inequality follows by the fact $\frac{\card{A}}{\card{B}} \le \frac{4}{3}$.
	This implies, with probability at least $1-o_n(1)$, $T$ itself is the largest subset of $T$. Furthermore, since $\card{T} \ge s$, we will extract $T$. In this case, we have already fail to recover any cluster.
\end{proof}

\section{Missing proof and discussion in Section~\ref{sec efficient}}

\subsection{Proof of Theorem~\ref{th app}} \label{apendix th5}

In this part, we discuss Theorem~\ref{th app} in detail. We will see where Theorem~\ref{th app} and the results in \cite{mathieu2010correlation} are different and why Theorem~\ref{th app} is true. 
To start with, we summarize the rounding step in Algorithm~\ref{alg one round} in the following algorithm. Let $T \subseteq V$ be a set of points and $F$ be a binary function over $T \times T$. We say a symmetric matrix $\hat{X}$ is good if 
\begin{itemize}
	\item $0 \le \hat{X}_{uv} \le 1$ for every $u,v \in T$,
	\item The distance between $\hat{X}$ and the feasible region of \eqref{pr SDP} is at most $1/\poly(\card{T})$,
	\item $d(\hat{X},F) \le d(X^*,F) + 1/\poly(\card{T})$, where $X^*$ is an optimal solution to \eqref{pr SDP}.
\end{itemize}

\begin{algorithm}[H]
	\caption{\textsc{SDPcluster}$(T,F)$ (Algorithm 2 in \cite{mathieu2010correlation})}\label{alg SDPcluster}
	\begin{algorithmic}
		\State Let $C=\emptyset$
		
		\State Let $\hat{X}$ be a good solution to \eqref{pr SDP}. 
		
		
		\While{$T \neq \emptyset$} \Comment{Use $\hat{X}$ to do rounding}
		
		\State Randomly select a point $T$ from $V$   
		
		\State Let $U = \{v\}$
		
		\For{$u \in T \setminus \{v\}$}
		
		\State Add $u$ to $U$ with probability $\hat{X}_{uv}$
		
		\EndFor
		
		\State $C \gets C \cup \{U\}$,  $T \gets T \setminus \{U\}$
		
		\EndWhile
		
		\State \Return $C$.
	\end{algorithmic}
\end{algorithm}

We remark that the only difference between Algorithm~\ref{alg SDPcluster} and Algorithm 2 in \cite{mathieu2010correlation} is that we use a good solution to do rounding, while they use an optimal solution to do rounding. Currently, we do not know a polynomial time algorithm that can solve general semi-definite programmings exactly. This is to say we do not know how to obtain an optimal solution to \eqref{pr SDP} in polynomial time.
Current theoretical guarantee for solving an SDP \cite{grotschel1981ellipsoid,alizadeh1995interior} is that for every $\epsilon \in (0,1)$, we can find a solution $\epsilon$-close to the feasible region with additive error at most $\epsilon$ in polynomial time. In our case, by choosing $\epsilon=1/\poly(n)$, this implies we can obtain a good solution to \eqref{pr SDP} in polynomial time via a naive rounding step to make sure the first condition in the definition of a good solution holds. This tiny change can ensure Algorithm~\ref{alg SDPcluster} definitely runs in polynomial time. In the following discussion, we will show Theorem~\ref{th app} is still true even if we do not use an optimal solution to do rounding.

\begin{theorem}\label{th approx}(Theorem 5 in \cite{mathieu2010correlation})  
	For every input $(T,F)$, let $\mathcal{A} = \textsc{SDPcluster}(T,F)$. 
	For every clustering function $C'$ over $T$, we have 
	\begin{align*}
		\E d(\mathcal{A},\hat{X}) \le 3 d(C',\hat{X}),
	\end{align*}
	where $\hat{X}$ is the good solution used in \textsc{SDPcluster}$(T,F)$.
\end{theorem}

The proof of Theorem~\ref{th approx} can be found in \cite{mathieu2010correlation}. Readers may notice that the statement of Theorem~\ref{th approx} is slightly different from the original statement in \cite{mathieu2010correlation}. In the original statement, Mathieu and Schudy, restricted $\hat{X}$ to be an optimal solution to \eqref{pr SDP}, while we relax this restriction to good solutions. We remark that, as Mathieu and Schudy claimed in their proof, as long as $\hat{X}$ is a symmetric matrix in $[0,1]^{\card{V} \times \card{V}}$, Theorem~\ref{th approx} holds. An immediately corollary of Theorem~\ref{th approx} is if we do rounding $\Omega(\log \frac{1}{\delta})$ times and pick $\mathcal{A}^*$ to be the clustering that is closest to $\hat{X}$, then with probability $1-\delta$, $d(\mathcal{A}^*,\hat{X}) \le 4 d(C',\hat{X})$. Next, we will see why $\mathcal{A}^*$ can achieve an additive error $O\left(\frac{\card{T}^{3/2}}{1-2p}\right)$ with high probability.

Let $M, N \in \R^{\card{T}\times\card{T}}$. We define $M\cdot N:= \sum_{u,v}M_{uv}N_{uv}$. For $F \in \{0,1\}^{\card{T}\times\card{T}}$, We define $\hat{F} \in \{-1,1\}^{\card{T}\times\card{T}}$ as follows:
\begin{align*}
	\hat{F}_{uv} = \begin{cases}
		& -1 \quad \text{if } F_{uv}=1 \\
		& 1 \quad \text{if } F_{uv}=0.
	\end{cases}
\end{align*}


\begin{claim}\label{cl p51}
	(Claim 16 in \cite{mathieu2010correlation})
	For every $M \in \{0,1\}^{\card{T}\times\card{T}}$ and $N \in [0,1]^{\card{T}\times\card{T}}$, we have
	\begin{align*}
		d(M,N) = \frac{1}{2}\left(\hat{M}\cdot N -\hat{M} \cdot M \right).
	\end{align*}
\end{claim}

Under the semi-random model, we define a symmetric random matrix $E \in \{0,1\}^{\card{T}\times\card{T}}$ in the following way. For every $(u,v)$ such that $u,v$ are in the same cluster in $V^*$, $E_{uv} = 0$ if and only if $(u,v)$ is corrupted. For every $(u,v)$ such $u,v$ are in the different clusters in $V^*$, $E_{uv} = 1$ if and only if $(u,v)$ is corrupted. Intuitively, $E_{uv} = \oracle(u,v)$ for every $u,v$, where the adversary in the oracle always gives the wrong answer.

\begin{claim}\label{cl p52}(Lemma 23 in \cite{mathieu2010correlation})
	Under the semi-random model, there is a constant $c>0$, such that with probability at least $1-4\exp(-\card{T})$, 
	\begin{align*}
		\abs{ \hat{E}\cdot X - \E \hat{E} \cdot X} \le c\card{T}^{\frac{3}{2}}
	\end{align*}
	for every symmetric matrix $X$ with trace at most $2\card{T}$ and smallest eigenvalue at least $-1/\poly(\card{T})$.
\end{claim}
We remark that the statement of Claim~\ref{cl p52} is slightly different from the statement of Lemma 23 in \cite{mathieu2010correlation}. In the original statement, Mathieu and Schudy didn't give a concrete bound of the probability of success and $X$ is forced to be positive semi-definite with trace to be $\card{T}$. We remark that every good solution satisfies the statement of the claim.
Here we give a short proof of the claim by slightly modifying the proof of Lemma 23 in \cite{mathieu2010correlation}.

\begin{cpf}
	Write $M=\hat{E} - \E \hat{E}$. Write $X = \sum_{i=1}^{\card{T}}\lambda_iv_iv_i^T$ by doing spectral decomposition of $X$. Without loss of generality, we assume that $\lambda_1 \ge \dots \lambda_r \ge 0 \ge \lambda_{r+1} \ge \lambda_{\card{T}} \ge -1/\poly(\card{T})$. Notice that 
	\begin{align*}
		\abs{M \cdot X} = \abs{\sum_{i=1}^{\card{T}}\lambda_iv_i^TMv_i} & \le \sum_{i=1}^{r}\lambda_i\abs{v_i^TMv_i}-\sum_{i=r+1}^{\card{T}}\lambda_i\abs{v_i^TMv_i} \\
		& \le \sum_{i=1}^{r}\lambda_i\rho(M)-\sum_{i=r+1}^{\card{T}}\lambda_i\rho(M) \\
		& = \sum_{i=1}^{\card{T}}\lambda_i\rho(M)-2\sum_{i=r+1}^{\card{T}}\lambda_i\rho(M) \\
		& \le (2\card{T}+\frac{1}{\poly(\card{T})})\rho(M) \le 3\card{T}\rho(M).
	\end{align*}
	It is sufficient to show $\rho(M) = O(\sqrt{T})$ with probability at least $1-4\exp(\card{T})$. We notice that $M$ is symmetric matrix whose entries on and above the diagonal are independent mean-zero sub-gaussian random variables. By Corollary 4.4.8 in \cite{vershynin2018high}, there is a constant $c$ such that $\rho(M) \le c\sqrt{T}$ with probability at least $1-4\exp(\card{T})$.

\end{cpf}

We know with probability $1-\delta$, $d(\mathcal{A}^*,\hat{X}) \le 4 d(C',\hat{X})$ for every clustering $C'$ over $T$. So we have
\begin{align*}
	d(\mathcal{A}^*,\bar{T})  \le d(\mathcal{A}^*,\hat{X}) + d(\hat{X},\bar{T}) 
	\le d(\mathcal{A}^*,\hat{X}) + d(T^*,\bar{T}) + \frac{1}{\poly(\card{T})}
	\le d(T^*,\bar{T}) + 4d(T^*,\hat{X})+\frac{1}{\poly(\card{T})},
\end{align*}
where the first inequality holds by triangle inequality, the second inequality holds since $\hat{X}$ is a good solution to SDP$(\bar{T})$, and $T^*$ is a feasible solution to SDP$(\bar{T})$. It remains to upper bound $d(T^*,\hat{X})$. It can be checked easily that 
\begin{align*}
	\E 
	\hat E = (1-2p) T^*.
\end{align*}
By Claim~\ref{cl p51}, we know 
\begin{align*}
	d(T^*,\hat{X}) & = \frac{1}{2} \left(\hat{T}^* \cdot \hat{X} - \hat{T}^* \cdot T^*\right) = \frac{1}{2(1-2p)} \left(\E \hat E \cdot \hat{X} - \E \hat E \cdot T^*\right) \\ 
	& = \frac{1}{2(1-2p)} \left( \E \hat E \cdot \hat{X} - \hat E \cdot \hat{X} + \hat E \cdot \hat{X} - \hat E \cdot T^* + \hat E \cdot T^*  -\E \hat E \cdot T^*    \right) \\
	& \le \frac{1}{2(1-2p)} \left( \hat E \cdot \hat{X} - \hat E \cdot T^* +2c\card{T}^{\frac{3}{2}}    \right) \\
	& = \frac{1}{2(1-2p)} \left( \hat E \cdot \hat{X} - \hat E \cdot E + \hat E \cdot E  - \hat E \cdot T^* +2c\card{T}^{\frac{3}{2}}    \right) \\
	& = \frac{1}{2(1-2p)} \left(2 \left(d(\hat{X},E)-d(T^*,E)\right)+2c\card{T}^{\frac{3}{2}}\right) \\
	& \le \frac{1}{2(1-2p)} \left(2 \left(d(\hat{X},\bar{T})-d(V^*,\bar{T})\right)+2c\card{T}^{\frac{3}{2}}\right) \\
	& \le \frac{c'\card{T}^{\frac{3}{2}}}{(1-2p)}.
\end{align*}
Here, the first inequality follows by Claim~\ref{cl p52}, the last equality follows by Claim~\ref{cl p51} and the last inequality holds because $\hat{X}$ is a good solution. To see why the second last inequality holds, we suppose that the adversary gets the chance to give a wrong label of $e=(u,v)$ but chooses to give the correct label. Then we have 
\begin{align*}
	\abs{T^*_e - E_e} = \abs{T^*_e - \bar{T}_e} + 1, 
\end{align*}
while
\begin{align*}
	\abs{\hat{X}_e - E_e} \le \abs{\hat{X}_e - \bar{V}_e} + 1,
\end{align*}
because $\hat{X}_e \in [0,1]$. So we know for every point pair $e$,
\begin{align*}
	\abs{\hat{X}_e - E_e} - \abs{T^*_e - E_e} \le \abs{\hat{X}_e - \bar{T}_e} - \abs{T^*_e - \bar{T}_e}.
\end{align*}
By sum all these inequalities over point pair $e$, we get the second last inequality. So far, we have shown with probability at least $1-\delta-4\exp(-\card{T})$, $d(\mathcal{A}^*,\bar{T}) \le d(V^*,\bar{T})+O\left(\frac{\card{T}^{3/2}}{1-2p}\right)$. 
In particular, since we do not need to solve \eqref{pr SDP} exactly, $\tilde{T}=\textsc{ApproxCorrelationCluster}(T)$ can be obtained in polynomial time.

\subsection{Missing technical theorem} \label{apendix th miss}
In this section, we prove the following technical theorem, which will be used to prove Theorem~\ref{th one round}.

\begin{theorem}\label{th tool}
	Let $V$ be a set of points such that $\card{V}=ts,$ where $t,s>0$. Let $c>0,\epsilon \ge 0$ be two numbers such that $c\left(1-2p\right)^2s^2/2>\epsilon$. Under the semi-random model, with probability at least $1-\exp\left(ts\log ts - c\left(1-2p\right)^3s^2/8\right)$, for every clustering function $V'$ over $V$ such that $d(V',V^*) \ge c\left(1-2p\right)s^2$, we have 
	\begin{align*}
		d(\bar{V},V') > d(\bar{V},V^*)+\epsilon,
	\end{align*}
	where $\bar{V}$ is the binary function over $V$ corresponding to the results that we query every point pair of $V$ and $V^*$ is the underlying clustering function of $V$.
\end{theorem}

We first introduce the following notations to simplify the proof. Let $V',\tilde{V}$ be two clustering functions over $V$. We let 
$D_{V'\tilde{V}}:=\{(u,v) \mid V'(u,v) \neq \tilde{V}(u,v)\}$
be set of point pairs that are labeled differently by $V'$ and $\tilde{V}$. In particular, for every clustering function $V'$, we define $\Dn_{V'V^*}=\{(u,v) \in D_{V'V^*} \mid V^*(u,v) \neq \bar{V}(u,v) \}$ and $\Dc_{V'V^*}=\{(u,v) \in D_{V'V^*} \mid V^*(u,v) = \bar{V}(u,v) \}$.

To prove Theorem~\ref{th tool}, we first prove the following lemma. 
\begin{lemma}\label{lm error}
	Let $V=[n]$ be a set of points. Let $V^*$ be the underlying clustering function of $V$. Let $\bar{V}$ be the binary function corresponding to the results that we query all point pairs of $V$. Let $V'$ be a clustering function over $V$. Then 
	\begin{align*}
		d(\bar{V},V^*) - d(\bar{V},V') = \card{\Dn_{V'V^*}}-\card{\Dc_{V'V^*}}
	\end{align*}
\end{lemma}

\begin{proof}
	Since $V^*(u,v) \neq V'(u,v)$ if and only if $(u,v) \in D_{V'V^*}$, we know 
	\begin{align*}
		d(\bar{V},V^*) - d(\bar{V},V') =  \sum_{(u,v) \in D_{V'V^*}} \left(\abs{\bar{V}(u,v)-V^*(u,v)} - \abs{\bar{V}(u,v)-V'(u,v)}\right).
	\end{align*}
	For every $(u,v) \in \Dn_{V'V^*}$, we have $\abs{\bar{V}(u,v)-V^*(u,v)} =1 $ and $\abs{\bar{V}(u,v)-V'(u,v)}=0$. On the other hand, for every $(u,v) \in \Dc_{V'V^*}$, we have $\abs{\bar{V}(u,v)-V^*(u,v)} =0 $ and $\abs{\bar{V}(u,v)-V'(u,v)}=1$.  Thus, we have
	\begin{align*}
		d(\bar{V},V^*) - d(\bar{V},V') = \card{\Dn_{V'V^*}}-\card{\Dc_{V'V^*}}.
	\end{align*}
\end{proof}

Now we use Lemma~\ref{lm error} to prove Theorem~\ref{th tool}.

\begin{proof} \textbf{of Theorem~\ref{th tool}} \ 
	We first fix a clustering function $V'$ over $V$ such that $d(V',V^*) \ge c\left(1-2p\right)s^2.$ We first show that with high probability, $d(\bar{V},V') > d(\bar{V},V^*)+\epsilon$. For every point pair $(u,v)$, we define random variable
	\begin{align*}
		x_{uv} = \begin{cases}
			& 1 \ \text{if $(u,v)$ is not corrupted} \\
			& -1 \ \text{otherwise}.
		\end{cases}
	\end{align*}
	We observe that for every realization of $\bar V$, we always have
	\begin{align}\label{eq corrupt}
		\card{\Dc_{V'V^*}} -\card{\Dn_{V'V^*}} \ge \sum_{e \in D_{V'V^*}}x_e.
	\end{align}
	This is because if an adversary gets a chance to output a wrong label of $e$, but does not do that, $\card{\Dn_{V'V^*}}$ will increase by $1$, while $\card{\Dc_{V'V^*}}$ will decrease by $1$.
	
	In expectation, we have 
	\begin{align}\label{eq expectation}
		\E \sum_{e \in D_{V'V^*}}x_e = \left(1-2p\right)\card{D_{V'V^*}} = \left(1-2p\right)d(V',V^*) \ge c\left(1-2p\right)^2s^2>2\epsilon.
	\end{align}
	Thus, we have 
	\begin{align*}
		\Pr\left( \card{\Dc_{V'V^*}} -\card{\Dn_{V'V^*}} \le \epsilon\right) & \le \Pr\left(\sum_{e \in D_{V'V^*}}x_e \le \epsilon\right)  \\
		& \le \Pr\left(\sum_{e \in D_{V'V^*}}x_e \le 
		\frac{\E \sum_{e \in D_{V'V^*}}x_e}{2}\right) \\
		& \le \exp\left(- \frac{\left(1-2p\right)^2 d(V',V^*)^2}{8 d(V',V^*)} \right) \\
		& \le \exp \left(-\frac{c\left(1-2p\right)^3s^2}{8}\right).
	\end{align*}
	Here, the first inequality follows by \eqref{eq corrupt}, the second inequality follows by \eqref{eq expectation}, the third inequality follows by the Hoeffding's inequality and in the last inequality, we use the assumption that $d(V',V^*) \ge c\left(1-2p\right)s^2.$ By Lemma~\ref{lm error}, we know that with probability most $\exp\left(c\left(1-2p\right)^3s^2/8\right)$,
	\begin{align*}
		d(\bar{V},V') = d(\bar{V},V^*) + \card{\Dc_{V'V^*}} -\card{\Dn_{V'V^*}} \le d(\bar{V},V^*) + \epsilon.
	\end{align*}
	
	Since the number of clustering function over $V$ is at most $(ts)^ {ts}$, we know that 
	\begin{align*}
		\Pr \left( \exists V', d(V',V^*) \ge c\left(1-2p\right)s^2, d(\bar{V},V') \le d(\bar{V},V^*) + \epsilon \right) & \le (ts)^ {ts} \exp \left(-\frac{c\left(1-2p\right)^3s^2}{8}\right) \\ 
		& = \exp \left(ts \log ts -\frac{c\left(1-2p\right)^3s^2}{8}\right).
	\end{align*}
	
	Thus, with probability at least $1-\exp\left(ts\log ts - c\left(1-2p\right)^3s^2/8\right)$, for every clustering function $V'$ over $V$ such that $d(V',V^*) \ge c\left(1-2p\right)s^2$, we have 
	\begin{align*}
		d(\bar{V},V') > d(\bar{V},V^*)+\epsilon.
	\end{align*}

\end{proof}

\subsection{Proof of Theorem~\ref{th one round}} \label{apendix th6}
The key part of the proof of Theorem~\ref{th one round}, is to show the following three claims.

\begin{claim}\label{cl pf61}
	In Algorithm~\ref{alg one round}, if there is some $i \in [k]$ such that $\card{T_i^*}>s_t$, but $h=0$, then $d(\tilde{T},T^*) > s^2_t/8$.
\end{claim}

\begin{cpf}
	Without loss of generality, we assume that $\card{T_1^*} > s_t$. We denote by $A_i:= \tilde{T_i} \cap T_1^*.$ Without loss of generality, we can assume $A_i \neq \emptyset$ if and only if $i \in [\ell]$, where $\ell$ is a positive integer. We say $\tilde{T}$ makes a negative mistake on $(u,v)$ if $T^*(u,v) = 1$ and $\tilde{T}(u,v)=0$. It is easy to see that the number of negative mistakes made by $\tilde{T_i}$ over $T_1^* \times T_1^*$ is 
	\begin{align*}
		\sum_{i=1}^{\ell}\sum_{j=i+1}^{\ell} \card{A_i}\card{A_j}.
	\end{align*}
	Since $h=0$, for every $i \in [\ell]$, $\card{A_i} \le s_t/2$.
	To lower bound the number of negative mistakes, we consider the following family of quadratic programming problems, parameterized by $\ell$.
	\begin{align}
		\label{pr QP}
		\tag{QP($\ell$)}
		\begin{split}
			\min \ & \sum_{i=1}^{\ell}\sum_{j=i+1}^{\ell} x_ix_j \\
			\stt \ & \sum_{i=1}^{\ell}x_i \ge s_t \\
			\ & 1 \le x_i \le \frac{s_t}{2} \quad \forall i \in [\ell] .\\
		\end{split}
	\end{align}
	Clearly every choice of $\{A_i\}_{i \in [\ell]}$ is corresponding to a feasible solution to \ref{pr QP}. Thus, we will show that for every $\ell \ge 2$, the optimal value of QP$(\ell)$ is at least $s^2_t/8$. 
	
	We prove this by induction. For the base case, it is easy to check the optimal value of QP$(2)$ is $s^2_t/4$. Now suppose that the optimal value of QP$(\ell)$ is at least $s^2_t/8$, we show this also correct for $\ell+1$. Let $x = \left(x_1,\dots,x_{\ell+1}\right)$ be a feasible solution to QP$(\ell+1)$. We consider two cases. 
	
	In the first case, there exist $i,j \in [\ell+1]$, such that $y=x_i+x_j \le s_t/2$. Without loss of generality, we can assume that $i=\ell,j=\ell+1$. Then we know that $x'=\left(x_1,\dots,x_{\ell-1},y\right)$ is a feasible solution to QP$(\ell)$. It can be checked that the objective value of $x$ is at least that of $x'$ and thus at least $s^2_t/8$.
	
	In the second case, for every $i,j \in [\ell+1]$, $x_i+x_j>s_t/2$. So we know there is some $i \in [\ell+1]$ such that $s_t/4<x_i \le s_t/2$. This implies the objective value of $x$ is at least $s^2_t/8$.
	
	Thus, by induction the number of negative mistakes is at least $s^2_t/8$. So we know 
	\begin{align*}
		d(\tilde{T},T^*) \ge \frac{s^2_t}{8},
	\end{align*}
	as long as $h=0$.
\end{cpf}

\begin{claim}\label{cl pf62}
	In Algorithm~\ref{alg one round}, if there is some $i \in [k]$ such that $\card{\tilde{T_i}}>s_t/2$ and $\tilde{T_i}$ is an $\eta$-bad set, where $\eta=1/4+p/2$, then $d(\tilde{T},T^*)>(1-2p)s_t^2/64$.
\end{claim}

\begin{cpf}
	We consider separately two cases.
	In the first case, we assume that for every $j \in [k]$, $\card{\tilde{T_i} \cap T^*_j} \le \card{\tilde{T_i}}/4$. Let $S: = \cup_{j=1}^{i^*} \tilde{T_i} \cap T^*_j$, where $i^*$ is the largest index such that $\card{S} \le \card{\tilde{T_i}}/2$. Thus we know $\card{\tilde{T_1}\setminus S} \ge \card{\tilde{T_i}}/2$. By the choice of $i^*$, we know that $\card{S} \ge \card{\tilde{T_i}}/4$.  So every point pair $(u,v)$ such that $u \in S$ and $v \in \tilde{T_i} \setminus S$ is labeled $1$ by $\tilde{T}$ but labeled $0$ by $T^*$. The total number of such point pairs is at least $\card{\tilde{T_i}}^2/8>s_t^2/32.$
	
	In the second case, we assume that there is some $j \in [k]$ such that $\card{\tilde{T_i} \cap T^*_j} > \card{\tilde{T_i}}/4$. We know $\card{\tilde{T_i} \setminus T^*_j } \ge (1-2p)\card{\tilde{T_i}}/4$, since $\tilde{T_i}$ is an $\eta$-bad set. We notice that every point pair $(u,v)$ such that $u \in \tilde{T_i} \cap T^*_j$ and $v \in \tilde{T_i} \setminus T^*_j $ is labeled $1$ by $\tilde{T}$ but labeled $0$ by $T^*$. The total number of such point pairs is at least $(1-2p)\card{\tilde{T_i}^2}/16>(1-2p)s_t^2/64.$
\end{cpf}

\begin{claim}\label{cl pf63}
	In Algorithm~\ref{alg one round}, if there is some $i,j,\ell \in [k]$, and $i \neq j$ such that $\card{\tilde{T_i}},\card{\tilde{T_j}}>s_t/2$ and $\tilde{T_i},\tilde{T_j}$ are both $(\eta,V^*_\ell)$-biased sets, then $d(\tilde{T},T^*)>(1-2p)s_t^2/16$.
\end{claim}

\begin{cpf}
	We notice that for every point pair $(u,v)$ such that $u \in \tilde{T_i} \cap V^*_\ell$ and $v \in \tilde{T_j} \cap V^*_\ell$, $(u,v)$ is labeled $0$ by $\tilde{T}$ but is labeled $1$ by $T^*$. The total number of such point pairs is at least $s_t^2/16$, since $\tilde{T_i},\tilde{T_j}$ are both $(\eta,V^*_\ell)$-biased sets and $\card{\tilde{T_i}},\card{\tilde{T_j}}>s_t/2$. Thus, we have $d(\tilde{T},T^*)>(1-2p)s_t^2/16$.
\end{cpf}

Now we are able to use the above claims to prove Theorem~\ref{th one round}.

\begin{proof}
	We first apply Theorem~\ref{th tool} on the sample set $T$ with $s=s_t,\epsilon = c_1 \left(ts_t\right)^{\frac{3}{2}}/\left(1-2p\right)$ and $c = 1/64$, where $c_1$ is a constant that satisfies Theorem~\ref{th app}.   We first show that the choice of parameter satisfies the statement of Theorem~\ref{th tool}. On the one hand, we have 
	\begin{align*}
		\epsilon = \frac{c_1 \left(ts_t\right)^{\frac{3}{2}}}{\left(1-2p\right)} = \frac{c_1 (c')^\frac{3}{2}t^6 \log^{\frac{3}{2}} n}{\left(1-2p\right)^{10}}.
	\end{align*}
	On the other hand, we have
	\begin{align*}
		\frac{c(1-2p)^2s_t^2}{2}=	\frac{\left(1-2p\right)^2s_t^2}{128} = \frac{(c')^2t^6 \log^2n}{128\left(1-2p\right)^{10}} >  \epsilon,
	\end{align*}
	because $c'$ is a large enough constant. So with probability at least
	\begin{align}\label{eq prob}
		1 - \exp \left(ts_t\log ts_t - \frac{\left(1-2p\right)^3s_t^2}{1024}\right) \ge 1- \exp \left(  -\left((c')^2-\frac{c'}{1024}\right)\frac{t^6 \log^{\frac{3}{2}}n}{\left(1-2p\right)^{9}} \right) \ge 1- \frac{1}{\poly(n)},
	\end{align}
	any clustering $T'$ such that $d(T',T^*)>(1-2p)s_t^2/64$ will satisfy
	\begin{align*}
		d(\bar{T},T') > d(\bar{T},T^*) + \frac{c_1 \left(ts_t\right)^{\frac{3}{2}}}{\left(1-2p\right)}.
	\end{align*}
	Here in \eqref{eq prob}, the first inequality follows by $\log ts_t \le \sqrt{ts_t}$ and the second inequality holds because $c'$ is a large enough constant.
	
	By Claim~\ref{cl pf61}, Claim~\ref{cl pf62} and Claim~\ref{cl pf63}, we know that if any one of the events in the statement of Theorem~\ref{th one round} does not happen, we will have $d(T',T^*)>(1-2p)s_t^2/64$. However,
	by Theorem~\ref{th app}, we know that with probability at least $1 -1/\poly(n)$, we have 
	\begin{align*}
		d(\bar{T},\tilde{T}) \le d(\bar{T},T^*) + \frac{c_1 \left(ts_t\right)^{\frac{3}{2}}}{\left(1-2p\right)}.
	\end{align*}
	This implies $d(\tilde{T},T^*) \le \left(1-2p\right)s^2_t/64$, with probability at least $1-1/\poly(n)$. And thus, the three events must happen together.
\end{proof}


\subsection{Proof of Theorem~\ref{th final}}\label{apendix th4}

We first prove the correctness of Algorithm~\ref{alg clustering}. Let $C=\{\tilde{V_1},\dots,\tilde{V_\ell}\}$ be the output of Algorithm~\ref{alg clustering}. We first show each element in $C$ is an underlying cluster. We know each $\tilde{V_i}= \{v \in V \mid \text{Test}(v,B_i) = \text{``Yes"}\}$. Also, we know $B_i \subseteq \hat{T_i},$ where $\hat{T_i} \in \textsc{ApproxCorrelationCluster}(T,1/\poly(n))$ in a certain stage of the algorithm and $\card{\hat{T_i}} \ge s_t/2.$ According to Theorem~\ref{alg one round}, we know with probability at least $1-1/\poly(n)$, $\hat{T_i}$ is an $(\eta,V^*_i)$-biased set, where $\eta=\frac{1}{4}+\frac{p}{2}$. By Hoeffding's inequality, by setting $\eta'=\frac{p+1}{3}$, we know 
\begin{align*}
	\Pr \left(B_i \text{ is not an }(\eta',V^*_i)\text{-biased set}\right) \le \exp \left(-2\card{B_i}\left(\frac{1-2p}{12}\right)^2\right) \le 1/\poly(n).
\end{align*}
So with probability at least $1-1/\poly(n)$, $B_i$ is an $(\eta',V^*_i)$-biased set. According to Lemma~\ref{lm test}, with probability at least $1-1/\poly(n)$, we have 
\begin{align*}
	\tilde{V_i}= \{v \in V \mid \text{Test}(v,B_i) = \text{``Yes"}\}=V^*_i \cap V = V^*_i.
\end{align*}
Here the last equality follows by the fact that $V^*_i \subseteq V$ at the time when $B_i$ is created. This is because no point in $V^*_i$ is put into other underlying clusters before $B_i$ is created. So each element in $C$ is an underlying cluster.

It remains to show every underlying cluster of size $\Omega\left(\frac{k^4\log n}{(1-2p)^6}\right)$ must be recovered with high probability. Suppose there is some underlying cluster $V^*_i$ such that $\card{V^*_i} \ge 2ks_{2k}=\Omega\left(\frac{k^4\log n}{(1-2p)^6}\right)$ not recovered by Algorithm~\ref{alg clustering}. 
Then at the end of the algorithm, $V^*_i \subseteq V$ and $\card{V}\ge 2ks_{2k}$.
Assume $\card{C} = h < k$. Then as long as $k-h \le t < 2k$, the sampled set $T$ of size $ts_{t}$ must contain some underlying cluster of size $\frac{t}{h-t}s_t \ge s_t$. By Theorem~\ref{th one round}, with probability at least $1-1/\poly(n)$, we will update $C$ again. However, at this time the algorithm has terminated. This gives a contradiction.
So every underlying cluster of size at least $O\left(\frac{k^4\log n}{(1-2p)^6}\right)$ must be recovered by Algorithm~\ref{th final}.

We next prove the sample complexity of Algorithm~\ref{alg clustering}. 
To show this, we first show that every time we invoke Algorithm~\ref{alg one round}, we must have parameter $t < 2k$. Suppose $t \ge 2k$, we know that $t/2 \ge k$. Since $T$ is partitioned into at most $k$ underlying clusters, we know that there must be at least one $i \in [k]$ such that $\card{T^*_i} \ge ts_{t/2}/2k \ge s_{t/2}$. According to Theorem~\ref{th one round}, with probability at least $1-1/\poly(n)$, $h>0$. In this case, we will not invoke Algorithm~\ref{alg one round} after updating $t/2$ by $t$. This implies every time we invoke Algorithm~\ref{alg one round}, we query $O\left(\card{T}^2\right)=O\left(\frac{k^8 \log^2n}{\left(1-2p\right)^{12}}\right)$ times and in each round we will call Algorithm~\ref{alg one round} $O(\log k)$ times. Since there are at most $k$ rounds, the number of queries we spend on Algorithm~\ref{alg one round} is $O\left( \frac{k^9\log k\log^2n}{\left(1-2p\right)^{12}}   \right)$. 

Next, we see each time we update $C$, we invoke Algorithm~\ref{alg test} at most $n$ times and each time we query $O\left(\frac{\log n}{\left(1-2p\right)^2}\right)$ times. Since we update $C$ at most $k$ times, the number of queries we spend on updating $C$ is $O\left(\frac{nk\log n}{(1-2p)^2}\right)$. So the query complexity of Algorithm~\ref{alg clustering} is 
$O\left( \frac{nk\log n}{\left(1-2p\right)^2} + \frac{k^9\log k \log^2n}{\left(1-2p\right)^{12}} \right)$.

$\hfill \square$


\section{Missing proof in Section~\ref{sec application}}

\subsection{Proof of Theorem~\ref{th fully}} \label{apendix th full}


It is sufficient to show with high probability $\bar{p} \ge p$ and $(1-2\bar{p}) = O\left((1-2p)\right)$, since in this case $\bar p$ is an appropriate upper bound of $p$ and we can use Algorithm~\ref{alg clustering} to solve the problem. We can assume $(1-2p)^4 \ge \frac{10 \log n}{n}$, because if $(1-2p)^4 < \frac{10 \log n}{n}$, there is no underlying cluster of size $\Omega\left(
\frac{k^4\log n }{\left(1-2p\right)^6}  \right)$ and Theorem~\ref{th fully} holds naturally.

We first analyze the set $A$.
For a given point $v$, denote by $V^*_v$ the underlying cluster that $v$ belongs to. We first show for every sampled set $A$, either there are two points $u,v$ in the same underlying cluster or there are two points $u,v$ such that $\card{V^*_u}+\card{V^*_v} \le \frac{n}{4}$. For simplicity, we say a such a point pair is good.
Let $S$ be the set of points $w$ such that $\card{V^*_w} \le \frac{n}{8}$. We know there are at most 8 underlying clusters that have size more than $\frac{n}{8}$. We can without loss of generality assume they are $V^*_1, \dots, V^*_i$, $i \le 8$. Since $\card{A} = 9$, we know there must be two points in $S$ or in the same underlying cluster. In the first case, we have $\card{V^*_u}+\card{V^*_v} \le \frac{n}{4}$, according to the definition of $S$. So $A$ must contain a good pair.

Next, we show with probability at least $1-1/\poly(n)$, $\bar{p} \ge p$. Denote by $\delta = 1-2p$ and $\bar{\delta} = 1 - 2\bar{p}$. For every $u,v \in A$ such that $u,v$ in same underlying cluster, we have 
\begin{align*}
	\E \text{count}_{uv} = 2p(1-p)\card{V} = 2p(1-p)n = \frac{1-\delta^2}{2}n \le \left(\frac{1}{2}-\frac{\delta^2}{4}\right)n.
\end{align*}
On the other hand, for every $u,v \in A$ such that $u,v$ in different underlying clusters, but $\card{V^*_u}+\card{V^*_v} \le \frac{n}{4}$, we have
\begin{align*}
	\E \text{count}_{uv} = 2p(1-p)\card{V} + (1-2p)^2\left(\card{V^*_u}+\card{V^*_v}\right) \le \left(\frac{1}{2}-\frac{\delta^2}{4}\right)n.
\end{align*} 
This implies if point pair $(u,v)$ is good, then 
\begin{align*}
	\left(\frac{1}{2}-\frac{\delta^2}{2}\right)n \le 	\E \text{count}_{uv} \le \left(\frac{1}{2}-\frac{\delta^2}{4}\right)n.
\end{align*}
In particular, the lower bound holds for every $u,v \in A$.

Let $(u,v)$ be a point pair in $A$, by Hoeffding's inequality, we know that 
\begin{align*}
	\Pr \left( \text{count}_{uv} \le \frac{1-4\delta^2}{2} n \right) \le \exp\left(-2 \left(\frac{3\delta^2}{2}n\right)^2\frac{1}{n} \right) = \exp\left(-\frac{9}{2}\delta^4n\right).
\end{align*}
By union bound, we know that 
\begin{align*}
	\Pr\left(M \le \frac{1-4\delta^2}{2} n \right) \le \Pr\left( \exists u \neq v \in A, \text{count}_{uv} \le \frac{1-4\delta^2}{2} n \right) \le n\exp\left(-\frac{9}{2}\delta^4n\right) \le \frac{1}{n^{44}}, 
\end{align*}
where the last inequality follows by $\delta^4 \ge \frac{10 \log n}{n}$.

So with probability at least $1-1/n^{44}$, we have 
\begin{align*}
	\bar{p}: = \frac{1}{2} -\frac{1}{4}\sqrt{1-\frac{2M}{n}} >  \frac{1}{2} -\frac{1}{4}\sqrt{1-\frac{2}{n}\frac{1-4\delta^2}{2}n} = \frac{1}{2}\left(1-\delta\right) = p,
\end{align*} 
where the inequality follows by $M>\frac{1-4\delta^2}{2} \sqrt{n}$. Since $\bar{p}>p$, and we know that the fully-random model with parameter $p$ is a special case of the semi-random model with parameter $\bar{p}$, we know that with probability at least $1-1/\poly(n)$, Algorithm~\ref{alg clustering} will recover all clusters of size $\Omega\left(
\frac{k^4\log n }{\left(1-2\bar{p}\right)^6}  \right)$ and the query complexity is  $O\left(\frac{nk\log n}{\left(1-2\bar{p}\right)^2} + \frac{k^9\log k\log^2n}{\left(1-2\bar{p}\right)^{12}} \right)$.

And it remains to show $\bar{p}$
is not too larger than $p$, so that we get the correct query complexity.
We will show that with probability at least $1-1/\poly(n)$, we have $\bar{\delta}>\delta/4$, which implies that $1/\left(1-2\bar{p}\right) \le 4/\left(1-2p\right)$. Let $u,v$ be a good point pair in $A$. We have
\begin{align*}
	\Pr\left(M  \ge \frac{1-\frac{1}{4}\delta^2}{2}n\right) \le \Pr\left(\text{count}_{uv}\ge \frac{1-\frac{1}{4}\delta^2}{2}n \right) \le \exp \left(-2\left(\frac{1}{8}\delta^2n\right)^2\frac{1}{n}\right) = \exp\left(-\frac{1}{32}\delta^4n\right). 
\end{align*}
Thus, with high probability we have
\begin{align*}
	\bar{\delta} = \frac{1}{2} \sqrt{1 - \frac{2M}{n}} > \frac{1}{2} \sqrt{1 - \frac{2}{n} \frac{1-\frac{1}{4}\delta^2}{2}n } = \frac{\delta}{4}.
\end{align*}

$\hfill \square$


\end{document}